\documentclass[12pt]{article}
  \usepackage[margin=0.86in]{geometry}  
  
  \usepackage{mathrsfs}  
\usepackage{kpfonts,comment} 

\usepackage{natbib}

\usepackage{graphicx,amsmath,amsfonts,amscd,amssymb,bm,epsfig,epsf,url,color,algorithm,algorithmic,enumerate}
\usepackage{amsthm,color,xcolor} 

\usepackage[utf8]{inputenc} 
\usepackage[T1]{fontenc}    
\usepackage{hyperref}       
\usepackage{url}            
\usepackage{booktabs}       
\usepackage{amsfonts}       
\usepackage{nicefrac}       
\usepackage{microtype}      
\usepackage{xcolor} 
\usepackage{array}

\newtheorem{Lemma}{Lemma}
\newtheorem{Proposition}{Proposition}
\newtheorem{Theorem}{Theorem}
\newtheorem{Definition}{Definition}

\newtheorem{Example}{Example}
\newtheorem{Remark}{Remark}
\newtheorem{Assumption}{Assumption}

\usepackage{tikz}

\newcolumntype{P}[1]{>{\centering\arraybackslash}p{#1}}

\usepackage{algorithmic,algorithm}
\usepackage{array}
\usepackage{enumerate}

\title{What is a Good Metric to Study Generalization of Minimax Learners?}
\author{Asuman Ozdaglar \and Sarath Pattathil \and  Jiawei Zhang\thanks{Corresponding author.  Authors are arranged in alphabetical order, and are affiliated with the Department of Electrical Engineering and Computer Science, Massachusetts Institute of Technology, Cambridge, MA, USA. \{asuman, sarathp, jwzhang, kaiqing\}@mit.edu. S.P. acknowledges support from MathWorks Engineering Fellowship. A.O and K.Z. were supported by MIT-DSTA grant 031017-00016. K.Z. also acknowledges  support  from Simons-Berkeley Research Fellowship.} \and Kaiqing Zhang} 
\date{} 
 
\begin{document}  
 
 \maketitle
 
 \begin{abstract}
 Minimax optimization has served as the backbone of many machine learning (ML) problems. Although the  {\it convergence behavior} of optimization algorithms has been extensively studied in the  minimax settings, their {\it generalization} guarantees in {stochastic minimax optimization problems}, i.e., how the solution trained on empirical data performs on unseen testing data, have been relatively underexplored.  A fundamental question remains elusive: {\it What is a good metric to study generalization of minimax  learners?} In this paper, we aim to answer this question by first showing that {\it primal risk}, a universal  metric  to study generalization in minimization problems, {which has also been adopted recently to study generalization in minimax ones,} fails in simple examples. 
 We thus propose a new metric to study generalization of minimax learners: the {\it primal gap}, {defined as the difference between the primal risk and its minimum over all models,} to circumvent the issues. 
 Next, we derive generalization {error}  bounds for the primal gap in nonconvex-concave settings. As byproducts of our analysis, we also solve two open questions: establishing generalization {error} bounds for primal risk and primal-dual risk{, another existing metric that is only well-defined when the global saddle-point exists,} in the strong sense, i.e., without strong concavity or assuming that the maximization and expectation can be interchanged, while either of these assumptions was needed in the literature. 
 Finally, we  leverage this new metric to compare the generalization behavior of two popular algorithms -- gradient descent-ascent (GDA) and gradient descent-max (GDMax) in stochastic minimax optimization.   
\end{abstract} 



\vspace{-5pt}

 \section{Introduction}\label{sec:intro}  

Stochastic minimax optimization, a classical and fundamental  problem in operations research and game theory, involves solving the following problem:
\begin{equation}
\min_{w\in W}\max_{\theta\in \Theta}E_{z\sim P_z}[f(w,\theta;z)]. \nonumber
\end{equation}
More recently, such minimax formulations have received increasing attention in machine learning, with significant applications in generative adversarial networks (GANs) \citep{goodfellow2014generative}, adversarial learning \citep{mkadry2017towards}, and reinforcement learning \citep{chen2016stochastic,dai2018sbeed}. Most existing works have focused  on the {\it optimization} aspect of the problem, i.e., studying the rates of convergence, robustness, and  optimality of  algorithms {for solving an empirical version of the problem where it approximates the expectation by an average over a sampled dataset}, in various minimax settings including convex-concave \citep{nemirovski2009robust, monteiro2010complexity}, nonconvex-concave \citep{lin2020gradient,rafique2018weakly}, and certain special nonconvex-nonconcave \citep{nouiehed2019solving,yang2020global} problems.  

However, the optimization aspect is not 
{sufficient to achieve} the success of  {stochastic}  minimax optimization in machine learning. In particular, as in classical supervised learning, which is usually studied as a {\it minimization} problem \citep{hastie2009elements}, the out-of-sample {\it generalization} performance is a key metric for evaluating the learned models. The study of generalization guarantees in minimax optimization (and related machine learning problems) has not received significant attention until recently  \citep{arora2017generalization,feizi2020understanding,yin2019rademacher,lei2021stability,farnia2021train,zhang2021generalization}. Specifically, existing works along this line have investigated two types of generalization guarantees:  {\it uniform}  convergence generalization bounds, and {\it algorithm-dependent} generalization bounds. The former is more general and irrespective  of the optimization algorithms being used, while the latter is usually  finer and   really explains what happens in practice, when optimization algorithms play an indispensable role.  In fact, the former might not  be able to explain generalization {performance}  {in deep learning, e.g., these bounds can increase with the training dataset size and easily become vacuous in practice  \citep{nagarajan2019uniform},} making the latter a more favorable metric for understanding the success of minimax optimization in machine learning.



Algorithm-dependent generalization for minimax optimization has been studied recently in \citep{farnia2021train,lei2021stability,xing2021algorithmic,yang2022differentially}. These papers build on the algorithmic stability framework developed in \citep{bousquet2002stability}{, which are further investigated in \citep{hardt2016train}}. In particular, these works have studied {\it primal risk} and/or (variants of)  {\it primal-dual risk} under different convexity and smoothness assumptions of the objective. Primal risk (see formal definition in \S\ref{sec:prelim}) is a natural extension of {the definition of risk} from minimization problems. Primal-dual risk, on the other hand, is defined similarly but based on the duality gap of the solution. It is know that it is {well-defined and can be optimized to zero only} when the {global} saddle-point exists (i.e., $\min$ and $\max$ can be interchanged). 
Based on these metrics, \citep{farnia2021train,lei2021stability} compare the performance of specific algorithms, e.g., gradient descent-ascent (GDA) and gradient descent-max (GDMax). We provide a more thorough literature review in   Section \ref{sec:related_works}. 

Although these metrics are natural extensions of generalization metrics from the {\it minimization} setting, they might not be the most suitable ones for studying generalization in stochastic  {\it minimax}  optimization, especially in the {\it nonconvex} settings that is pervasive in machine/deep learning applications, where the global saddle-point might not exist. In particular, we are interested in  the following fundamental question:
\begin{center}
	{\it What is a good metric to study generalization of minimax learners\footnote{Hereafter, we use {\it learner} and {\it learning algorithm} interchangeably.}?}  
\end{center}
 In this paper, we make an initial attempt  to answering this question, by identifying the inadequacies of {the existing metric,}  
 and proposing a new metric, the {\it primal gap} that overcomes these inadequacies. We then provide generalization error bounds for the newly proposed metric, and discuss how it captures information not included in the other existing metrics. We  summarize our contributions as follows.

 \vspace{-6pt} 
 
 \paragraph{Contributions.} First, we introduce an example through which we identify the inadequacies of {\it primal risk}, a well-studied metric for generalization in stochastic minimax optimization, in capturing the generalization behavior of {\it nonconvex-concave} minimax problems. Second, to address the issue, we propose a new metric -- the {\it primal gap}, which provably avoids the issue in the example, and derive its generalization error bounds. Next, we leverage this new metric to compare the generalization behavior of GDA and GDMax, two popular algorithms for minimax optimization and GAN training, and answer the question of {\it when does GDA generalize better than GDMax?} Moreover, we also address two open questions in the literature: establishing generalization error bounds  for primal risk and primal-dual risk  without strong concavity or assuming that the maximization and expectation can be interchanged, while at least one of these assumptions was needed in the literature  \citep{farnia2021train,lei2021stability,xing2021algorithmic,yang2022differentially}.  Finally, under certain assumptions of the max learner, our results also generalize to the nonconvex-nonconcave  setting.  

\begin{table}
\begin{center}
\begin{tabular}{ | P{5.2cm} | P{2.2cm}| P{1.3cm} | P{6.2cm} |} 
  \hline
  Reference & Assumption & Metric & Rate \\ 
  \hline
  \citep{farnia2021train} & NC-$\mu$-SC & PR & $L\sqrt{\kappa^2 + 1} \epsilon$ \\ 
  \hline
  \citep{lei2021stability}& NC-$\mu$-SC & PR & $L (1 + \kappa) \epsilon$ \\ 
  \hline
  \citep{lei2021stability}& $\mu$-SC-SC & PD & $\sqrt{2} L(1+ \kappa) \epsilon $\\  
  \hline
  This work (Theorem \ref{main3})& NC-C  & PG & $\sqrt{4L\ell C_p^2}\cdot \sqrt{\epsilon} + \epsilon L +4L_\theta^*C_e/\sqrt{n}$ \\ 
  \hline
  This work (Lemma \ref{main1})  & NC-C & PR & $\sqrt{4L\ell C_p^2}\cdot \sqrt{\epsilon} + \epsilon L$ \\ 
  \hline
  This work (Theorem \ref{thm:pd_cc})& C-C  & PD & $ \left( \sqrt{4L\ell C_p^2} +  \sqrt{4L\ell (C_p^w)^2} \right) \sqrt{\epsilon} + 2\epsilon L$ \\ 
  \hline
\end{tabular}
\vspace{4pt} 
\caption{Generalization bounds for $\epsilon$-stable algorithms. PR stands for Primal Risk, PD stands for the primal-dual risk and PG stands for the primal gap. NC-$\mu$-SC stands for nonconvex-$\mu$ strongly concave. $\mu$-SC-SC stands for $\mu$ strongly convex-$\mu$ strongly concave. NC-C stands for nonconvex-concave. C-C convex-concave. $L$ is the Lipschitz constant of the function $f$. $\kappa$ stands for the condition number $L/\mu$. The constants in the in the theorems have been defined in the appropriate sections. Note that there are other results in \citep{farnia2021train, lei2021stability} for cases where the expectation and $\max$ operator can be interchanged. This case is almost identical to the minimization problem and we thus  do not include it in the table.}
\end{center}
\end{table}

\subsection{Related work}
\label{sec:related_works} 

\paragraph{Algorithms for minimax optimization.} There is a vast literature on algorithms for minimax optimization. The most popular algorithms include the Extragradient (EG), the Optimistic Gradient Descent Ascent (OGDA) and the Gradient Descent Ascent  and their variants. The EG algorithms introduced in \citep{korpelevich1976extragradient}, has been analyzed in several papers including \citep{monteiro2010complexity, mokhtari2020unified, mokhtari2020convergence, golowich2020last} for (strongly)convex-(strongly)concave problems. Another popular algorithm is OGDA introduced in \citep{popov1980modification} and has been analyzed in several recent works including \citep{daskalakis2017training, hsieh2019convergence, golowich2020tight}. Once again, all these works focus on the (strongly)convex-(strongly)concave setting. Stochastic versions of these algorithms in similar settings have also been analyzed in several papers including \citep{nemirovski2009robust, hsieh2019convergence, fallah2020optimal}. A few papers including \citep{lin2020gradient, zhang2020single, huang2022efficiently, zhang2021complexity, ostrovskii2021efficient, kong2019complexity, zhang2020single} analyze gradient based algorithms in the nonconvex-(strongly)concave cases. Some papers including \citep{rafique2018weakly, yang2021faster, ostrovskii2021nonconvex, grimmer2020landscape} analyze special cases of nonconvex-nonconcave (like nonconvex-PL) for algorithms like GDA and its variants. However, in this paper, we are interested in the generalization performance of these algorithms. We summarize below the most related literature that studies the generalization behavior in minimax optimization problems. 

\paragraph{Algorithm-independent  generalization.} Specific to the machine learning problems of GAN and adversarial training, there have been several papers studying the uniform convergence generalization bounds. \citep{arora2017generalization} establish a uniform convergence generalization bound which depends on the number of discriminator parameters. \citep{wu2019generalization} connect the stability-based theory  to differential privacy (\citep{shalev2010learnability}) in GANs and numerically study the generalization behavior in GANs.  \citep{zhang2017discrimination,bai2018approximability} analyze  the Rademacher complexity of the players to show the  uniform convergence bounds for GANs. In the simpler Gaussian setting, \citep{feizi2020understanding} and \citep{schmidt2018adversarially} derive bounds for GANs and adversarial training, respectively. The uniform convergence bounds for adversarial training have also been studied under several statistical learning frameworks, e.g., PAC-Bayes \citep{farnia2018generalizable}, Rademacher complexity \citep{yin2019rademacher}, margin-based \citep{wei2019improved}, and VC analysis \citep{attias2019improved}. 
Recently, \citep{zhang2021generalization}  investigate the generalization of empirical  saddle point (ESP) solution in strongly-convex-concave problems using a stability-based approach. Note that these results are not specific to the optimization algorithms being used.

\paragraph{Algorithm-dependent generalization.} Algorithm specific generalization bounds  for minimax optimization have attracted increasing attention. Based on the algorithmic stability framework in  \citep{bousquet2002stability}, \citep{farnia2021train} have established generalization bounds of standard gradient descent-ascent and proximal point algorithms under the convex-concave  setting, and those of stochastic GDA and GDMax under the nonconvex-strongly concave setting. Concurrently, \citep{lei2021stability} derive  high-probability generalization  bounds for both convex-concave and weakly convex-weakly concave settings, with possibly nonsmooth objectives, also through  the lens of algorithmic stability.  Both works hinged on the metrics of {\it primal risk} and {\it primal-dual risk}. As shown in the present work, the former is not necessarily suitable to characterize the generalization behavior of minimax optimization, while the latter  {is known to be appropriate only when the saddle point exists, which is usually not the case in the nonconvex settings that are common in machine learning.} 
Following this line of work, \citep{xing2021algorithmic} provide generalization bounds specifically for  adversarial training, which is essentially the primal risk, also using the algorithmic stability framework. Recently,  \citep{yang2022differentially}  study  the generalization of stochastic GDA under differential privacy constraints.


\section{Preliminaries}\label{sec:prelim}


\subsection{Problem formulation}
In this paper, we consider the following (stochastic) minimax problem:
\begin{equation}
\label{problem:og_min_max}
\min_{w\in W}\max_{\theta\in \Theta}E_{z\sim P_z}f(w,\theta;z).
\end{equation}

We make the following assumption on the sets $W$ and $\Theta$ throughout the paper.
\begin{Assumption}
$W$ and $\Theta$ are convex, closed sets, and we further assume that $W$ is compact with $\|w\|\le M(W), \forall w\in W$. Here $M(W)$ is a constant dependent on the set $W$.
\end{Assumption}
  
Let $r(w,\theta)=E_{z\sim P_z}f(w,\theta;z)$.   
For a training dataset $S=\{z_1,\cdots,z_n\}$ with $n$ i.i.d. variables drawn from $P_z$, we define $r_S(w,\theta)=\frac{1}{n}\sum_{i=1}^nf(w,\theta;z_i)$. Next, we define the following quantity:

\begin{Definition}[Primal risk (empirical/population)]
\textbf{Primal population risk} is given by\footnote{Note that we slightly abuse the notation here by allowing $r$ and $r_S$ to have inputs that can be both $w$ and $(w,\theta)$. The distinction will be clear from context.} 
$$r(w)=\max_{\theta\in \Theta}E_{z\sim P_z}f(w,\theta;z),$$ 
and the \textbf{primal empirical risk} is given by: 
$$r_S(w)=\max_{\theta\in \Theta}\frac{1}{n}\sum_{i=1}^nf(w,\theta;z_i).$$  
\end{Definition} 

Throughout this paper, we use $(w_S,\theta_S)$ to denote a solution of the minimax  problem: $\min_{w\in W}\max_{\theta\in \Theta}r_S(w,\theta)$. Notice that $(w_S,\theta_S)$ need not be a global saddle-point of $r_S$. Furthermore,  we use $(w^*,\theta^*)$ to denote a solution of $\min_{w\in W}\max_{\theta\in \Theta}r(w,\theta)$. Once again, notice that $(w^*,\theta^*)$ may not be a saddle point of $r$.

The goal in Problem \eqref{problem:og_min_max}   is to minimize the primal population risk $r(w)$. Note that this function can be decomposed as 
\begin{align}
\label{eq:generalization_decomp}
r(w)=r_S(w)+(r(w)-r_S(w)).
\end{align}

In practice, we only have access to $r_S(w,\theta)$, and our goal is to design algorithms for minimizing $r(w)$ using dataset $S$.
Suppose $A$ is a learning algorithm initialized at $(w,\theta)=(0,0)$. We define $(w_S^A,\theta_S^A)$ to be the output of Algorithm $A$ using dataset $S$.

From Equation \eqref{eq:generalization_decomp}, it is clear if we ensure $r_S(w_S^A)$  as well as $r(w_S^A)-r_S(w_S^A)$ are small, this would guarantee that  $r(w_S^A)$ is small, which is the goal of Problem \eqref{problem:og_min_max}. Note that we can always ensure that $r_S(w_S^A)$ is small by using a good optimization Algorithm $A$ (if the problem is tractable). The main goal in the study of generalization is therefore to estimate the generalization error of the primal  risk, as defined below. 

\begin{Definition} 
\label{def:exp_risk}
The generalization error for the primal risk is defined as:
\begin{align}
\zeta_{gen}^P(A)=E_SE_A[r(w_S^A)-r_S(w_S^A)].
\end{align}
Here the expectations are taken over the randomness in the dataset $S$, as well as any randomness used in the Algorithm $A$.
\end{Definition}

This metric has been used to study generalization in stochastic minimization problems, i.e., when the maximization set $\Theta$ is a singleton, as well as several recent works in stochastic minimax optimization (see \citep{hardt2016train, farnia2021train,lei2021stability}). 

We are interested in the question of  when the solution to the empirical problem $w_S^A$ has good {\it generalization behavior}, i.e., when  $E[ r(w_S^A) - \min_{w\in W}  r(w)]$ is small -- $w_S^A$ is an approximate minimizer of the primal population risk $r$. In the next subsection, we briefly describe why the generalization error of the primal risk $\zeta_{gen}^P(A)$ {is a good measure to study the generalization behavior} in minimization problems. 

\subsubsection{$\zeta_{gen}^P(A)$ for minimization problems} 

Consider a stochastic  optimization problem of the form 
\begin{align}
\label{problem:minimization}
\min_{w\in W} \ E_{z\sim P_z} [g(w ;z)]. 
\end{align}
We define the (minimization) primal risk (population and empirical version respectively) as: $r(w) = E_{z\sim P_z} g(w ;z)$, and $r_S(w) =\frac{1}{n}\sum_{i=1}^n g(w;z_i)$. The generalization error $\zeta_{gen}^{P, min}(A)$ for the (minimization) primal risk is the same as in Definition \ref{def:exp_risk} using the (minimization) primal risk.

Assume that the generalization error of the primal risk for an Algorithm $A$ is small, say $\zeta_{gen}^{P, min}(A) \leq \epsilon$. This implies that (from Definition \ref{def:exp_risk}): $
E[r(w_S^A)] \leq E[r_S(w_S^A)] + \epsilon$. 
Note that the expectation is with respect to $S$ and $A$. Now, in order to show that $w_S^A$ has  good generalization behavior, we first see that:
\begin{align}
E[ r(w_S^A) - \min_{w\in W}  r(w)] \leq E[r_S(w_S^A)] + \epsilon - \min_{w\in W}  r(w). 
\end{align}
However, note that for minimization problems, since $E[r_S] = r$, we have that\footnote{Here we use the fact that $E_z[\min_x f(x,z)] \leq \min_x E_z[f(x,z)]$.}   
$\min_{w \in W} r(w) \geq E[\min_{w \in W} r_S(w)]$, 
which gives us:
\begin{align*}
E[ r(w_S^A) - \min_{w\in W}  r(w)] \leq  E[r_S(w_S^A)] + \epsilon - E[\min_{w \in W} r_S(w)] =  E[r_S(w_S^A) - \min_{w \in W} r_S(w)] + \epsilon 
= \epsilon.
\end{align*} 

Therefore, for minimization problems, if the generalization error for primal risk is small, the solution to the empirical risk minimization problem has good generalization behavior. Next, we highlight some results in the literature which discusses  generalization error bounds of the primal risk. These results depend on the concept of algorithmic stability we use later. 

\subsection{Stability of algorithms}

Stability analysis is a powerful tool to analyze the generalization behavior of algorithms (see \citep{bousquet2002stability}). In this section, we will review some definitions and theoretical results about stability  bounds existing in the current literature. More specifically, in this paper, we adopt the following definition of stability:

\begin{Definition}[$\epsilon$-stable Algorithm] 
Suppose that $A$ is a randomized  algorithm for solving the stochastic  minimax problem. We define $(w_S^A,\theta_S^A)$ as the output of Algorithm $A$ using dataset $S$. We say $S$ and $S'$ are neighboring dataset if they defer only in one sample.
An Algorithm $A$ is defined to be $\epsilon$-stable if $E_A\|w_S^A-w_{S'}^A\|\le \epsilon$ and $E_A\|\theta_S^A-\theta_{S'}^A\|\le \epsilon$ for any neighboring datasets $S$ and $S'$.
\end{Definition}
\citep{hardt2016train} gives the following basic result for the generalization error of $r_S(w)$.  

\begin{Theorem}[\citep{hardt2016train}]\label{hardt}
Consider the (stochastic) minimization problem  defined in \ref{problem:minimization}. Suppose $g(\cdot;z)$ is ${\bar{L}}$-Lipschitz continuous, i.e., {$\forall z$}, it holds that $\| g(w_1; z) - g(w_2; z) \| \leq \bar{L} \|w_1 - w_2 \|, \forall w_1, w_2 \in W$. 
Then, for an $\epsilon$-stable Algorithm $A$, we have 
$|E_SE_A[r(w_S^A)-r_S(w_S^A)]|\le \bar{L}\epsilon$.  
\end{Theorem}


\subsubsection{{When is primal risk a valid  metric for minimax learners?}}

According to the above discussions for minimization problems, we know that the primal risk is a valid  metric to study generalization behavior in these problems, and furthermore, the generalization error bound of the primal risk can be estimated in terms of algorithmic stability. 
However, Theorem \ref{hardt} cannot be directly extended to analyze the generalization behavior of minimax learners because we have an additional maximization step before taking expectation.

A natural question emerges: Under what conditions does primal risk serve as a valid  metric to study generalization behavior of minimax problems. 
One sufficient condition is when the maximization step and expectation can be interchanged, i.e., when
$$\max_{\theta\in \Theta}E_{z\sim P_z}f(w,\theta;z)=E_{z\sim P_z}[\max_{\theta\in \Theta}f(w,\theta;z)]$$
for any distribution $P_z$. 
Letting $f_{\max}(w;z)=\max_{\theta\in \Theta}f(w,\theta;z)$,  we  further have
\begin{eqnarray*}
r(w)=\max_{\theta\in \Theta}E_{z\sim P_z}f(w,\theta;z)=E_{z\sim P_z}[\max_{\theta\in \Theta}f(w,\theta;z)]=E_{z\sim P_z}f_{\max}(w;z).
\end{eqnarray*}
Therefore, the minimax problem in \eqref{problem:og_min_max} is equivalent to the (stochastic) minimization problem with loss function $f_{\max}(w;z)$. Moreover, letting $P(S)$ be the uniform distribution over the dataset $S=\{z_1,\cdots,z_n\}$, we have 
\begin{eqnarray*}
r_S(w)=\max_{\theta\in\Theta}E_{z\sim P(S)} [f(w,\theta;z)]E_{z\sim P(S)} [\max_{\theta\in \Theta} f(w,\theta;z)]=\frac{1}{n}\sum_{i=1}^nf_{\max}(w;z_i).
\end{eqnarray*}
Therefore, $r_S(w)$ is just the empirical primal risk corresponding to the minimization problem with loss function $f_{\max}(w;z)$.
Hence, Theorem \ref{hardt} can be directly used to minimax problems where the maximization and expectation can be interchanged.

\begin{Theorem}\label{exc}
Suppose that $f(w,\theta;z)$ is $\bar{L}$-Lipschitz continuous with respect to $w$, i.e., $|f(w_1,\theta;z)-f(w_2,\theta;z)|\le \bar{L}\|w_1-w_2\|$ for any $w_1,w_2\in W, \theta \in \Theta$ and $z$. If an Algorithm $A$ is $\epsilon$-stable, we have $$E_SE_A[r(w_S^A)-r_S(w_S^A)]\le \bar{L}\epsilon.$$
\end{Theorem}
\begin{proof}
From the previous analysis along with Theorem \ref{hardt}, it suffices to show that $f_{\max}(\cdot;z)$ is $\bar{L}$-Lipschitz continuous.
In fact, we have
\begin{align*} 
f_{\max}(w_1;z)-f_{\max}(w_2;z)& =f(w_1,\theta(w_1);z)-f(w_2,\theta(w_2);z)\\ 
&\le f(w_1,\theta(w_1);z)-f(w_2,\theta(w_1);z)\le  \bar{L}\|w_1-w_2\|, 
\end{align*}
where $\theta(w)\in \arg\max_{\theta\in \Theta}f(w,\theta;z)$, the first inequality is because of the definition of $\theta(w)$ and the second inequality is because of the Lipschitz continuity of $f$ with respect to $w$.
Using the same argument, we can prove
$$f_{\max}(w_2; z)-f_{\max}(w_1;z)\le \bar{L}\|w_1-w_2\|.$$
Therefore, we prove the $\bar{L}$-Lipschitz continuity of $f_{\max}(\cdot;z)$ and hence finish the proof.
\end{proof}

By the above discussion, we know that if maximization and expectation can be interchanged, the minimax problem can be reduced to a minimization problem and hence the primal risk is a valid metric for studying the generalization behavior of minimax learners and the generalization error can be estimated using the same method as for minimization problems.
In practice, the adversarial-training problems can be such an example of minimax problems. 

\begin{Example}[Adversarial-training]\label{adv}
We consider the adversarial training problem  \citep{mkadry2017towards}.  Suppose we have loss function $g(w;z)$ for a supervised learning problem. Here $z$ denotes the training sample and $w$ denotes the model parameter. 
Due to the noise in the data or due to an adversarial attack, for any sample $z$, we consider an uncertainty  set $B(z,\epsilon_0)$ around it. The goal is to train a model that is robust to the data with possible perturbation in the uncertainty set. 
Let $\theta_z$ be some adversarial sample from the set $B(z,\epsilon_0)$ and let $\theta$ be an infinite dimensional vector (functional) with the component $\theta_z$  corresponding to the sample  $z$. 
Define the function $\iota_{B}(v)$ to be the indicator function of the set $B$, i.e., $\iota_B(v)=0$ if $v\in B$ and $\iota_B(v)=\infty$ otherwise.
The goal of adversarial training is to solve the following minimax problem:
\begin{equation}
\min_{w}\max_{\theta}~~E_{z\sim P_z}f(w,\theta;z),
\end{equation}
where $f(w,\theta;z)=g(w;\theta_z)+\iota_{B(z,\epsilon_0)}(\theta_z)$. 
For any distribution $P_z$  over $z$'s, we have
\begin{align*}
\max_{\theta}~E_{z\sim P_z}f(w,\theta;z)&=\max_{\theta}~E_{z\sim P_z}[g(w;\theta_z)+\iota_{B(z,\epsilon_0)}(\theta_z)]=E_{z\sim P_z}[\max_{\theta_z}~(g(w;\theta_z)+\iota_{B(z,\epsilon_0)}(\theta_z))]\\
&=E_{z\sim P_z}[\max_{\theta}~f(w,\theta;z)], 
\end{align*} 
where the second and the third equalities use the fact that $\theta_{z'}$ does not contribute to $f(w,\theta;z)$ if $z\ne z'$.
Therefore, the expectation and maximization can be interchanged in adversarial training problems. This implies that the results of Theorem \ref{exc} can be applied and therefore primal risk is a {valid} metric to study the generalization behavior in such problems.
\end{Example}

Unfortunately, maximization and expectation are not necessarily interchangeable for many minimax problems. 
If they are not interchangeable,  it is unclear how to estimate the generalization error bound of the primal risk. In fact, whether primal risk is still a good metric for studying generalization behavior in such problems remains elusive.

In the next section, we will see how to estimate generalization error bound of primal risk for nonconvex-concave and even nonconvex-nonconcave problems. To the best of our knowledge, this is the first result which provides generalization error bounds for the primal risk without assuming the interchangeability 
 or strong concavity of the inner maximization problems (see e.g.,  \cite{lei2021stability}). 
Furthermore, we will see that even in some simple minimax problems, the generalization error bound of  the primal risk can fail to capture the generalization behavior of minimax learners. We then propose a {new metric   and use  its generalization error to properly  characterize the generalization behavior of minimax learners.}     

 \section{Primal Gap: A New Metric to Study  Generalization}\label{sec:metric} 
  
The key idea behind the success of $\zeta_{gen}^P(A)$ as {a way to characterize}   
to  study generalization for minimization learners is   that $E[r_S(w)] = r(w)$ for any $w$,  
which is no longer the case in the minimax case. In fact,  we first show via example that a good bound for the generalization error of primal risk does not imply good generalization behavior for minimax learners. 

\subsection{{Primal risk can fail for  minimax learners}} 
\label{subsec:example}

We provide an example where the generalization error of the primal risk is small, but the final solution to the empirical problem has poor generalization behavior. In this example, the minimizer of $r_S(w)$ is suboptimal for $r(w)$ with high probability, and $E_S[r(w_S)-r(w^*)]$ is large. 

\begin{Example}[Analytical example] 
\label{example:main_ex}
Let $y\sim N(0,1/\sqrt{n})$ be a Gaussian random variable in $\mathbb{R}$. Define the truncated Gaussian variable $z\sim P_z$ as follows: $z=y$ if $|y|<\lambda \log n/\sqrt{n}$ and $z=\lambda \log n/\sqrt{n}$ if $y\ge \lambda \log n/\sqrt{n}$.
Let $f(w,\theta;z)=\frac{1}{2}w^2- \left( \frac{1}{2n^2}\theta^2-z\theta+1 \right) w$, where $w\in W=[0,1]$, $\theta\in \Theta = [-\lambda n,\lambda {n}]$ with a sufficiently large $\lambda>0$, and $z_i\sim P_z$ be i.i.d truncated Gaussian variables.
Then, we have $r_S(w,\theta)=\frac{1}{2} w^2- \left( \frac{1}{2n^2}\theta^2-\frac{\sum_{i=1}^nz_i}{n}\theta+1 \right) w$,  
and
\begin{equation}\label{rw}
r(w,\theta)=\frac{1}{2} w^2- \left( \frac{1}{2n^2}\theta^2+1 \right) w.
\end{equation}
Note that this leads to the primal population risk function: $r(w) = \frac{1}{2}w^2 - w$. 

It is not hard to see that  we always have $r_S(w)\ge r(w)$. Note that this means $\zeta_{gen}^P(A) \leq 0$, and thus we have {a small generalization error for primal risk}.  
However, we can prove that for large enough $\lambda$,  
\begin{equation}\label{pop}
E_S[r(w_S)-r(w^*)]\ge 0.02.
\end{equation}
This means that $w_S$ has a constant error compared to $w^*$  in terms of the population risk, despite that its generalization error is small.  
This phenomenon is due to that $\min_{w\in W}r_S(w)-\min_{w\in W}r(w)>c$ for some $c>0$, and hence minimizing $r_S(w)$ is very different from minimizing $r(w)$.  
\end{Example}

This example shows that the generalization error of primal risk is not a good measure to study generalization in minimax learners. The main drawback is that $\min_{w}r_S(w)$ and $\min_wr(w)$ can be very different. We now introduce another more practical example, from GAN training, to further illustrate this point.

\begin{Example}[GAN-training example]\label{example:gan}
Suppose that we have a real distribution $P_r$ in $\mathbb{R}^d$ which can be represented as $G^*(y)$ with $y\in \mathbb{R}^k$ drawn from  a standard Gaussian distribution $P_0$ and a mapping $G^*:\mathbb{R}^k\rightarrow \mathbb{R}^d$. For an arbitrary generator $G$, we define $P_G$ to be the distribution of the random variable  $G(y)$ with $y\sim P_0$. So our goal is to find a generator $G$ such that $P_G=P_r$. GAN is a popular tool for solving this problem.
Consider a GAN with generator $G$, parametrized by $w$ and discriminator $D$ parametrized by $\theta$.  The goal of GAN training is to find a pair of a generator $G$ and a discriminator $D$ that solves the minimax problem:
\begin{align}
&\min_{G}\max_{D}~~~\{E_{x\sim P_r}\phi(D(x))+E_{x\sim P_G}[\phi(1-D(x))]\} \nonumber \\
&\qquad \qquad \qquad =\min_{w}\max_{\theta}~~~\{E_{x\sim P_r}\phi(D_{\theta}(x))+E_{y\sim P_0}[\phi(1-D_{\theta}(G_w(y)))]\}, \nonumber
\end{align}
where $\phi:\mathbb{R}\rightarrow \mathbb{R}$ is concave, monotonically increasing and $\phi(u)=-\infty$ for $u\le 0$.  To connect to the minimax formulation in \eqref{problem:og_min_max}, we note that $z=(x,y)$, and $P_z=P_r\times P_0$. Also, we denote
$$
r(w, \theta) = E_{x\sim P_r}\phi(D_{\theta}(x))+E_{y\sim P_0}[\phi(1-D_\theta(G_w(y)))]
$$ 
to be the population risk. We now give the empirical version of this problem. Let $S_1=\{x_1,\cdots,x_n\}$ and $S_2=\{y_1,\cdots,y_n\}$. Let $S=S_1\cup S_2$ and $r_S(w,\theta)=\frac{1}{n} \left( \sum_{i=1}^n \phi(D_\theta(x_i)+\phi(1-D_\theta(G_w(y_i))) \right)$. 
We assume that $P_{G_w}$ has the same support set as $P_r$. Moreover, we assume that $\|w-w^*\|\le 0.5$  and $G_w(y)$ is $1$-Lipschitz w.r.t. $w$ for any $y$. Here $w^*$ denotes the parameter for which $G_{w^*} = G^*$. Then, combining  Theorem B.1 in \citep{arora2017generalization} and the Lipschitz continuity of $G_w(y)$ as well as $\|w-w^*\|\le 0.5$, we have that the distance between the sets $S_1$ and $\{ G_w(y_1), G_w(y_2), \cdots, G_w(y_n) \}$ will be larger than $0.6$ with probability greater than $1 - {O}(n^2/e^{d})$.  
Now, if $n$ is only of polynomial size of $d$, the optimal discriminator for disjoint datasets outputs $1$ on one dataset, and $0$ on the other. On the other hand, when $w = w^*$, the optimal discriminator for the population problem outputs $1/2$ for any sample it receives. Combining these two results, we have: 
$$E_S[\min_{w\in W}r_S(w)-\min_{w\in W}r(w)]\ge (1-\delta)\left(2\phi(1)-2\phi(1/2)\right)$$
which is bounded away from $0$.
\end{Example}

 Note that in this example, we also have $E_S[\min_wr_S(w)-\min_wr(w)]>0$, implying that {using}  
$\zeta_{gen}^P(A)$ might not be a good {way to characterize the generalization behavior in GAN training}. To address this issue, we next define a new metric, the primal gap, {and use its generalization error} to study the generalization of minimax learners.   

\subsection{Primal gap {to the rescue}} 

The population and empirical versions of the primal gap are defined as follows:
\begin{Definition}[Primal gap (empirical/population)]
The  \textbf{population primal gap} is defined as $$\Delta(w)=r(w)-\min_{w\in W}r(w),$$  
and the \textbf{empirical primal gap} is defined as $$\Delta_S(w)=r_S(w)-\min_{w\in W}r_S(w).$$
\end{Definition}
Notice that these two primal gaps can always take $0$ at $w_S\in\arg\min_{w\in W}r_S(w)$ and $w^*\in\arg\min_{w\in W}r(w)$ respectively  even if the saddle point of problem \eqref{sec:prelim} does not exist.
Next, we define the expected generalization error of this primal gap as follows:
\begin{Definition}
\label{def:exp_pg}
The generalization error for the primal gap is defined as  $$\zeta_{gen}^{PG}(A)=E_SE_A[\Delta(w_S^A)-\Delta_S(w_S^A)].$$
\end{Definition}

\begin{Remark}\label{remark:interchange}
For Example \ref{adv}, since the maximization and expectation can be interchanged, the minimax problem is equivalent to a minimization problem. Then we have
\begin{align*}
E_S[\min_wr_S(w)]&= E_S[\min_w\max_{\theta}E_{z\sim P_z(S)}f(w,\theta;z)]=E_S[\min_wE_{z\sim P_z(S)}[\max_{\theta}f(w,\theta;z)]]\\
&=E_S[\min_wE_{z\sim P_z(S)}[f_{\max}(w;z)]]\le E_S[E_{z\sim P_z(S)}[f_{\max}(w;z)]] 
\end{align*}
for any $w$. 
Therefore, we have $E_S[\min_wr_S(w)]\le \min_wr(w)$. Consequently, we have $\zeta_{gen}^P\ge \zeta_{gen}^{PG},$ 
which means that good generalization bounds for  the primal risk implies good generalization bounds for the primal gap. Therefore, if the maximization and expectation are interchangeable,  primal risk is sufficient to study the generalization behavior because the generalization error of the primal risk is an upper bound of the generalization error of the primal gap in this case.
\end{Remark}

Now we provide bounds on $\zeta_{gen}^{PG}(A)$ for stable algorithms $A$, and show that in Example \ref{example:main_ex}, $\zeta_{gen}^{PG}(A)$ cannot be small (unlike $\zeta_{gen}^{P}(A)$).

\subsection{Relationship between  generalization and stability}
\label{sec:stable_algo}

We provide bounds for the generalization error of the primal gap (Definition \ref{def:exp_pg}) for $\epsilon$-stable Algorithm $A$. We will focus on the nonconvex-concave case where the following assumptions are made throughout the rest of the paper. 
  
\begin{Assumption}
\label{ass:noncvx_conc}
The function $f$ in Problem \eqref{problem:og_min_max} is nonconvex-concave, i.e., $f(w, \cdot; z)$ is a concave function for all $w\in W$ and for all $z$.
\end{Assumption}
Next we define the notion of {\it capacity}, which will play a key role in the bounds we derive for $\zeta_{gen}^{PG}(A)$.
\begin{Definition}[Capacity]
\label{def:capacity}
For any $w\in W$ and any constraint set $\Theta$, we define 
\begin{align}
\Theta(w)=\arg\max_{\theta\in \Theta}r(w,\theta) \qquad \Theta_S(w)=\arg\max_{\theta\in \Theta}r_S(w,\theta). \nonumber
\end{align} 
We define the capacities $C_p$ and $C_e$   as:
\begin{align}
C_p(\Theta) = \max_{w\in W}\mathrm{dist}(0,\Theta(w)),\qquad C_e(\Theta) = \max_{S}\max_{w\in W}\mathrm{dist}(0,\Theta_S(w)), \nonumber
\end{align}
{where $\mathrm{dist}(p,\mathcal{S})$ denotes the distance between a point $p$ to a set $\mathcal{S}$ in Euclidean space, i.e., $$\mathrm{dist}(p,\mathcal{S}):=\inf_{q\in\mathcal{S}}\|p-q\|_2.$$   
For the specific  constraint set in Problem \eqref{problem:og_min_max}, we {succinctly} denote the capacities as $C_p$ and $C_e$, respectively.} 
\end{Definition}

The norm of the  {model} parameter {(its distance to $0$)} 
is usually viewed as the metric for the complexity of the model. In fact, the norm of the optimal solution determines the Rademacher complexity of the function class in statistical learning theory \citep{vapnik1999overview}. Moreover, in deep learning, minimum-norm solution of overparameterized neural networks is well-known to enjoy better generalization performance \citep{zhang2021understanding}. Hence, we view the capacity constant  $C_e$ and $C_p$ as natural metrics to capture the model complexity  for the best response of the max learner{, i.e., the power of the maximizer},  when using the empirical data set and population data  respectively. 


Now, we are ready to discuss the relationship between the stability bound and the generalization error of algorithms in nonconvex-concave minimax problems. All proofs have been deferred to the appendix. We make the following assumptions throughout the paper:
\begin{Assumption}\label{Lipschitz-smooth}
The gradient of $f$ is $\ell$-Lipschitz-continuous for all $z$, i.e., for all $z$
\begin{align}
\| \nabla f(w_1, \theta_1; z) - \nabla f(w_2, \theta_2; z) \| \leq \ell (\| w_1 - w_2 \| + \|\theta_1 - \theta_2 \|) , ~~~\forall w_1, w_2 \in W,~~~\forall \theta_1, \theta_2 \in \Theta.  \nonumber 
\end{align}
Moreover, fixing $w\in W$, the partial gradient $\nabla_\theta f(w,\cdot;z)$ is $\ell_{\theta\theta}$-Lipschitz continuous with respect to $\theta$ for all $z$, i.e., $\| \nabla_{\theta} f(w, \theta_1; z) - \nabla_{\theta} f(w, \theta_2; z) \| \leq \ell_{\theta \theta} \|\theta_1 - \theta_2 \| , \forall w \in W, \ ~~~\forall \theta_1, \theta_2 \in \Theta$. 
\end{Assumption}

\begin{Assumption}\label{Lipschitz-continuous}
For any $\Theta_1\subseteq \Theta$, 
we assume that $f$ is $L(\Theta_1)$-Lipschitz-continuous 
with respect to $w\in W,\theta\in \Theta_1$ for all $z$, i.e., $\|f(w_1, \theta_1; z) - f(w_2, \theta_2; z) \| \leq L(\Theta_1)(\| w_1 - w_2 \| + \|\theta_1 - \theta_2 \|) , \quad \forall w_1, w_2 \in W, \  \forall \theta_1, \theta_2 \in \Theta_1$, 
 and the gradient $\nabla f(w,\theta;z)$ is uniformly bounded as $\|\nabla_{w,\theta}f(w,\theta;z)\|\le L(\Theta_1)$ for all $z$ and $w\in W, \theta\in \Theta_1$. Moreover, $f(w^*,\cdot;z)$ is $L_\theta^*$-Lipschitz continuous with respect to $\theta$ where $w^*\in \arg\min_{w\in W}r(w)$. We also define $L :=L(B(0,2C_p+1)\cap \Theta)$ and $L_r := L(B(0,r)\cap \Theta)$, where $B(v,r)$ denotes the $l_2$-ball with radius $r$ centered at $v$. 
\end{Assumption}


%

Note that we can decompose the generalization error of the primal gap as follows:
\begin{align}
\zeta_{gen}^{PG}(A) &:= E_SE_A[\Delta(w_S^A)-\Delta_S(w_S^A)] = E_SE_A[r(w_S^A) - r_S(w_S^A)] + E_S[\min_{w\in W}r_S(w)-\min_{w\in W}r(w)] \nonumber \\
&= \zeta_{gen}^{P}(A) + E_S\big[\min_{w\in W}r_S(w)-\min_{w\in W}r(w)\big]. \nonumber
\end{align}

Next, we provide a bound on the generalization error for the primal risk $\zeta_{gen}^{P}(A)$. To the best of our knowledge, this is the first  bound for $\zeta_{gen}^{P}(A)$ in the nonconvex-concave (without strong concavity) setting. 

\begin{Lemma}\label{main1}
The generalization error of the primal risk of an $\epsilon$-stable Algorithm $A$ for a minimax problem with concave maximization problem can be bounded by $\zeta_{gen}^P(A) \leq \sqrt{4L\ell C_p^2}\cdot \sqrt{\epsilon} + \epsilon L$. 
\end{Lemma}


Since we already have the generalization error for the primal risk $E_SE_A[r(w_S^A)-r_S(w_S^A)]$ from Lemma \ref{main1}, we only need to estimate
\begin{align}
E_SE_A\big[\min_{w\in W}r_S(w)-\min_{w\in W}r(w)\big]=E_S\big[\min_{w\in W}r_S(w)-\min_{w\in W}r(w)\big] \qquad \text{[Primal Min Error]}.
\label{eq:second_term_pg}
\end{align}

The following theorem gives the generalization bound of the primal gap using the upper bound from Lemma \ref{main1} and bounding the Primal Min Error in Equation \eqref{eq:second_term_pg}. 
\begin{Theorem}\label{main3}
Suppose Algorithm $A$ is $\epsilon$-stable. The generalization error  bound of the primal gap is given by
$$\zeta_{gen}^{PG}(A)\le \sqrt{4L\ell C_p^2}\cdot \sqrt{\epsilon} + \epsilon L +4L_\theta^*C_e/\sqrt{n}.$$
\end{Theorem}

The first term in the bound above is from the generalization bound of the primal risk, as shown in Lemma \ref{main1}. {Note that the bound in Lemma \ref{main1} only involves $C_p$, as the key in the analysis is to upper-bound the population risk $r(w_S^A)$, which requires bounding the power of the maximizer using the population capacity $C_p$. This reflects the intuition that the power of the maximizer should affect the generalization behavior of minimax learners, and the stronger the maximizer is, the harder for the learner to generalize. On the other hand, the bound in Theorem \ref{main3} additionally involve $C_e$, the empirical capacity. Technically, $C_e$ (instead of $C_p$) appears since we need to bound $\min_w r_S(w)$ (defined on the empirical dataset) in the Primal Min Error term in \eqref{eq:second_term_pg}. The appearance of $C_e$ reflects the intuition that the difference between the maximizers of the empirical and population risks should make a difference in characterizing the generalization of minimax learners. This intuition cannot be captured by the generalization error of the primal risk, as in Lemma \ref{main1}. 
Note that in the minimization case, the Primal Min Error can be upper-bounded directly by zero, and such a distinction disappears, making primal risk a valid metric.}   

\subsection{Revisiting Example   \ref{example:main_ex}}
Recall Example \ref{example:main_ex} in  Section \ref{subsec:example}. In this example, we have that the primal risk {has a small generalization error,}  but the solution $w_S$ does not {generalize well}. 
In particular, as shown in the appendix (Proposition \ref{example-min-gap}), we have 
\begin{equation}\label{11}
E_S[\min_{w\in W}r_S(w)-\min_{w\in W}r(w)]\ge 0.005.
\end{equation}
 
On the other hand, it is easy to compute that $L_\theta^*=\lambda \log n/\sqrt{n}$ and $C_e=\lambda n$. Therefore, by Theorem \ref{main3}, we have an upper bound for the Primal Min Error (see Equation \eqref{eq:second_term_pg}): $E_S[\min_{w\in W}r_S(w)-\min_{w\in W}r(w)]\le 4L_\theta^*C_e/\sqrt{n}=4 \log n,$ 
which is tight up to a $\log$ factor according to \eqref{11}. Therefore, the primal gap has a constant generalization error which is consistent with the observation that the solution to the empirical problem does not have good generalization behavior.   

\vspace{-5pt}

\subsection{Nonconvex-nonconcave case} 

In this section, we extend our results to the nonconvex-nonconcave setting. We will show that under certain assumptions on the inner maximization problem, we can derive generalization error bounds for the primal risk and primal gap in terms of algorithmic stability.

We make the following assumptions on the inner maximization problem:
\begin{Assumption}\label{oracle1}
For any $\gamma>0$, there exists an algorithm which outputs $\theta^{\gamma}_P(w)$, for the inner maximization problem $\max_{\theta\in \Theta}r(w,\theta)$, satisfying the following conditions:
\begin{enumerate}
\item $r(w)-r(w,\theta_P^{\gamma}(w))\le \gamma$.
\item $\|\theta_P^{\gamma}(w)-\theta_P^{\gamma}(w')\|\le \frac{\lambda_p}{\gamma}\|w-w'\|$ with some constant $\lambda_p>0$ for all $w,w'\in W$. 
\end{enumerate}
\end{Assumption}
\begin{Assumption}\label{oracle2}
For any $\gamma>0$, there exists an algorithm which outputs $\theta^{\gamma}_{E}(S)$, for the inner maximization problem $\max_{\theta\in \Theta}r_S(w^*,\theta)$, satisfying the following conditions:
\begin{enumerate}
\item $r_S(w^*)-r_S(w^*,\theta_E^{\gamma}(S))\le \gamma$. 
\item For any neighboring  dataset $S,S'$, 
we have $\|\theta_E^{\gamma}(S)-\theta_E^{\gamma}(S')\|\le \frac{\lambda_e}{n\gamma}$ with some constant $\lambda_e>0$. 
\end{enumerate}
\end{Assumption}

The following lemma gives sufficient conditions for these two assumptions to hold.
\begin{Lemma}\label{sufficient}
Consider constants $D_e \geq \gamma$ and $D_p\geq \gamma$.
\begin{enumerate}
\item Suppose that gradient ascent  with diminishing stepsizes $c_0/t$ for the problem $\max_{\theta\in \Theta}r(w,\theta)$ has convergence rate $r(w)-r(w,\theta^s)\le D_p/s$. Then we define $\theta_p^{\gamma}(w)$ by performing $s=D_p/\gamma$ steps of gradient ascent. Then, $\theta_p^{\gamma}(w)$ satisfies Assumption \ref{oracle1}.
\item Suppose that gradient ascent  with constant stepsize $c_0$ for the problem $\max_{\theta\in \Theta}r(w,\theta)$ has convergence rate $r(w)-r(w,\theta^s)\le D_p\eta^s$ for some constant $0 < \eta < 1$. Then we define $\theta_p^{\gamma}(w)$ by $s=\log (D_p/\gamma)/\log(1/\eta)$ steps of gradient ascent. Then,  $\theta_p^{\gamma}(w)$ satisfies Assumption \ref{oracle1}. 
\item Suppose that gradient ascent  with diminishing stepsizes $c_0/t$ for the problem $\max_{\theta\in \Theta}r_S(w,\theta)$ has convergence rate $r_S(w)-r_S(w,\theta^s)\le D_p/s$. Then we define $\theta_e^{\gamma}(S)$ by performing $s=D_e/\gamma$ steps of gradient ascent. Then,  $\theta_e^{\gamma}(S)$ satisfies Assumption \ref{oracle2}.
\item Suppose that gradient ascent  with constant stepsize $c_0$ for the problem $\max_{\theta\in \Theta}r_S(w,\theta)$ has convergence rate $r_S(w)-r_S(w,\theta^s)\le D_e\eta^s$ for some constant $0 < \eta < 1$. Then we define $\theta_e^{\gamma}(w)$ by  $s=\log (D_e/\gamma)/\log(1/\eta)$ steps of gradient ascent. Then, $\theta_e^{\gamma}(S)$ satisfies Assumption \ref{oracle2}.
\end{enumerate}
\end{Lemma}
\begin{Remark}
Note that for some practical nonconvex optimization problems in machine learning, gradient descent indeed converges to the global minima  at a reasonably  fast rate, e.g., in training deep overparametrized neural networks \citep{du2019gradient}, robust least squares problems \citep{el1997robust}, phase retrieval and matrix completion \citep{ma2019implicit}. 
{Our Assumptions \ref{oracle1} and \ref{oracle2} can be viewed as an abstract summary of some benign properties of gradient descent for certain nonconvex optimization problems.}  
\end{Remark} 

Furthermore, we assume that $f(\cdot,\cdot;z)$ is $L$-Lipschitz\footnote{Note that this is different from the $L$ defined for the nonconvex-concave case. Here $L$ captures the Lipschitz constant over the whole constraint set. In the nonconvex-concave case,  $L=L(B(0,2C_p+1))$.}    
continuous in $W\times \Theta$.
This, along with Assumptions \ref{oracle1} and  \ref{oracle2}, allows us to derive the generalization error bounds of the primal risk and primal gap in terms of algorithmic stability.
\begin{Lemma}\label{ncncp}
Suppose that Assumption \ref{oracle1} holds. If a minimax learning Algorithm $A$ is an $\epsilon$-stable algorithm, 
we have 
$$\zeta_{gen}^P(A)\le L\epsilon+\sqrt{L\lambda_p}\sqrt{\epsilon}.$$
\end{Lemma}
Similarly, we can derive the generalization bound for the primal gap given the above assumptions.
\begin{Theorem}\label{ncncpg}
Suppose Assumptions \ref{oracle1} and \ref{oracle2} hold. Then we have
$$\zeta_{gen}^{PG}(A)\le \zeta_{gen}^P(A)+\sqrt{L\lambda_e}\big/\sqrt{n}.$$
\end{Theorem}

The proof of this theorem is similar to the proof of Lemma \ref{ncncp} and Theorem \ref{main3} and hence omitted.

\vspace{-4pt}

 \section{Comparison of GDA and GDMax}\label{sec:comparison}

In Section \ref{sec:stable_algo}, we provide generalization bounds for the primal gap for any $\epsilon$-stable algorithm. In this section, we focus on two algorithms in particular -- GDA and GDMax. These two algorithms are described in Algorithms \ref{alg_gda} and \ref{alg_gdmax} in Appendix \ref{appendix:gda_gdmax}.

We note that though analyzing the {\it optimization} properties of GDA/stochastic GDA for solving the empirical minimax problem is an important topic, our focus in this paper is on studying the generalization behavior of these algorithms. We assume that the empirical version of the stochastic minimax problem can be solved by GDA and GDMax, i.e., we assume that GDA and GDMax satisfy the following assumption:

\begin{Assumption}\label{convergence_GDA}
Let $A$ be a minimax learner, such as GDA or GDMax. Then we assume that $A$ has the following convergence rate: $E_A[r_S(w^t)-\min_{w\in W}r_S(w)]\le (\phi_{A}(M(W))+\phi_{A}(C_e))/\psi_{A}(t)$,   
where $M(W)$ is the maximum of the norms of $w$, and $\phi_{A}(s)$, $\psi_{A}(s)$ are nonnegative, increasing functions that tend to infinity  as $s\rightarrow \infty$.
\end{Assumption} 

For simplicity, throughout this section, we assume that $\|f(w,\theta;z)\|\le 1$ for all $w,\theta$, and $z$. The next theorem provides a bound for the population primal gap $\Delta(w_S^A) := r(w_S^A)-\min_{w\in W}r(w)$. Note that the goal of any algorithm is to make this gap as small as possible.

For an Algorithm $A$ and  subsets $W_0\subseteq W, \Theta_0\subseteq \Theta$, we define $A(W_0,\Theta_0)$ as the algorithm which restricts $A$ to solve  \eqref{problem:og_min_max} under constraint sets $W_0$ and $\Theta_0$. Specifically, $A(W,\Theta)$ is just $A$. 

\begin{Theorem} 
\label{populated_primal_gap}
Let $w_S^{A,t},\theta^{A,t}_S$ be the $t$-th iterate generated by Algorithm $A$  using dataset $S$.
Assume that $\{\theta^{A,t}_S\}\subseteq \Theta_0=  \Theta_\theta^A$ for $t\le T$ with probability $1-\delta$ (due to the randomness in $S$) and $B(0,C_p)\subseteq \Theta_\theta^A$. Here $B(v,r)$ denotes the $l_2$-ball with radius $r$ centered at $v$. Let $A_0=A(W,\Theta_0)$. Then after $T$ iterations of Algorithm $A$, the population primal gap can be bounded as:
$$E_S[r(w^{A,T}_S)-\min_{w\in W}r(w)]\le  \underbrace{(\phi_{A_0}(M(W)) + \phi_{A_0}(C_e(\Theta_\theta^A)))/\psi_{A_0}(T)+ 4L_\theta^*C_e(\Theta_\theta^A)/\sqrt{n}}_{{II}} + \underbrace{\zeta_{gen}^P(A_0)}_{{I}} + \delta,$$
where $\zeta_{gen}^P(A_0)=E_SE_A[r(w_S^{A_0,T})-r_S(w_S^{A_0,T})]$ is the generalization error of the primal risk of Algorithm $A_0$.
\end{Theorem}

\begin{Remark}
Theorem \ref{populated_primal_gap} builds a closer connection between generalization behavior and the dynamics of the minimax learner $A$.  
It shows that suitable restriction to the max learner can lead to better minimax learner, in terms of generalization. We make this clear in the comparison of GDA and GDMax by analyzing the three terms in Theorem \ref{populated_primal_gap}. 
\end{Remark}

\vspace{-4pt}

\subsection{Analyzing the term $I$}
First, we study the generalization error bound of the primal risk, i.e., $\zeta_{gen}^P$ in Theorem \ref{populated_primal_gap}.
For GDA, we can estimate $\zeta_{gen}^P$ by using Lemma \ref{main1}. Therefore, it suffices to estimate the stability of GDA. We do this in the following lemma:
\begin{Lemma}
\label{lemma:gda_stab}
Let $c_0=\max\{\alpha_0,\beta_0\}$, If we use diminishing stepsizes $\alpha_t = \alpha_0/t$ and $\beta_t = \beta_0/t$ for GDA for $T$ iterations, we have the stability bound $\epsilon^{GDA} \leq 2L_{\Theta_\theta^{GDA}}T^{c_0\ell}/(n\ell)$.
\end{Lemma}

{Now, since we have a bound for $\zeta_{gen}^P(A)$ for $\epsilon$-stable Algorithm $A$ in Lemma \ref{main1}, we can substitute the stability bound for GDA from Lemma \ref{lemma:gda_stab} in this expression to get a bound on $\zeta_{gen}^P(GDA)$ for GDA. We do this in the next proposition.} 
We can bound $\zeta_{gen}^P(A_0)$ for GDA by substituting the stability bound in Lemma \ref{lemma:gda_stab} into  Lemma \ref{main1} (letting $\epsilon=\epsilon^{GDA}$). 

\begin{Proposition}
\label{eq:primal_risk_GDA}
Let $c_0=\max\{\alpha_0,\beta_0\}$ and assume that $f(\cdot,\cdot;z)$ is $L_{\Theta_\theta^{GDA}}$-Lipschitz-continuous inside the set $W\times \Theta_\theta^{GDA}$. 
For GDA with diminishing stepsizes $\alpha_0/t,\beta_0/t$ run for $T$ iterations (denoted by $GDA_T$), the generalization error of the primal risk can be bounded by:
$$\zeta_{gen}^P(GDA_T)\le (L_{\Theta_\theta^{GDA}})^{3/2}\sqrt{8  C_p^2/\ell}\sqrt{T^{c_0\ell}/n} + 2L_{\Theta_\theta^{GDA}}^2T^{c_0\ell}/(n\ell).$$
\end{Proposition}

However, for GDMax, we can not compute a uniform stability bound that vanishes as $n$ goes to infinity. In fact,
 we can show from the following simple example that $\zeta_{gen}^P(\text{GDMax})$ can be a constant that is independent of $n$, which means that for the case where $r(w,\theta)$ is nonconvex-concave, the generalization error of  primal risk of GDMax can be undesirable.
 
\begin{Example}[Constant generalization error of primal risk for GDMax]
\label{example2}
Consider a dataset $S$ with $n$ elements. Define the objective function:
$f(w,\theta;z)=\left( \frac{w}{n^2}-z \right) \theta- \frac{\theta^2}{2n},$
where $w\in W=[-n \sqrt{n},n \sqrt{n}]$, $\theta\in \Theta=\mathbb{R}$ and $z$ is drawn from the uniform distribution over $\{-1/\sqrt{n},1/\sqrt{n}\}$. 
We have
$$r_S(w)=\frac{n^2}{2} \left( \frac{w}{n^2} -\frac{1}{n} \sum_{i=1}^nz_i \right)^2,$$
and $r(w)=\frac{w^2}{2n^2}$. 
Therefore, $\min_{w\in W}r(w)=0$. From the definition of the function $f$ and the sets $W$ and $\Theta$, we have $\ell=1/n^2$, $L=\mathcal{O}(1/\sqrt{n})$. 

Note that one step of GDMax can attain the minimizer of $r_S(w)$ (since it is a one dimensional quadratic problem), i.e., $w_S=n\sum_{i=1}^nz_i$ and $r_S(w_S) = 0$. Furthermore, we have $E_S r(w_S)  = E[\frac{(\sum_{i=1}^nz_i)2}{2}] = 1/2 > 0$. Thus, $\zeta_{gen}^P(\text{GDMax}) =  E[r(w_S) - r_S(w_S)] = 1/2 > 0$ cannot be made small.  
\end{Example}

Therefore, from Proposition \ref{eq:primal_risk_GDA} and Example \ref{example2}, we see that the bound for the expected population primal gap contains the term $\zeta_{gen}^P$ which cannot be bounded for GDMax, whereas can be bounded for GDA which leads us to the conclusion that GDA generalizes better than GDMax for such problems. However, it is possible to bound $\zeta_{gen}^P(\text{GDMax})$ in certain problems, and in this case the other  terms in Theorem \ref{populated_primal_gap} become crucial. We analyze them next.

\subsection{Analyzing the term $II$}\label{subsec:C_e_C_p}

As shown in Example \ref{example:main_ex}, sometimes GDMax can have a good generalization bound for the primal risk. Therefore, we need to analyze the other two terms in Theorem \ref{populated_primal_gap}, i.e., $(\phi_A(M_w)+ \phi_A(C_e(\Theta_\theta^A)))/\psi_A(T)$ and $L_\theta^*C_e(\Theta_\theta^A)/\sqrt{n}$. For these two terms, since $L_\theta^*$ is fixed, the constant $C_e(\Theta_\theta^A)$ is the key term which differentiates the performance of different algorithms.

By definition, the constant $C_e(\Theta_\theta^{GDMax})$ for GDMax is nearly $C_e$ (See Definition \ref{def:capacity}). Therefore, the population primal gap after $T$ steps of GDMax is dominated by $C_e$ if $C_e$ is large. However, the set $\Theta_\theta^{GDA}$ for GDA can be much smaller than $\Theta$, which implies that $C_e(\Theta_\theta^{GDA})$ can be much smaller than $C_e$. This phenomenon  can be seen from Example \ref{example:main_ex}: If we perform one step of GDMax with primal stepsize $1$, we can attain $w^1=w_S$. Then $E_S[r(w^1_S)-\min_{w\in W}r(w)]\ge 0.005$ from \eqref{pop}.  For GDA, we can see that $w^1=1$ after one step of GDA with stepsize $1$. Therefore, GDA generalizes better than GDMax. Generally, we have the following estimate of $C_e(\Theta_\theta^{GDA})$.
\begin{Lemma}\label{gda_bounded}
Let $L_0=\max_{z}\|\nabla f(w_0,\theta_0;z)\|$. Let $c_0=\max\{\alpha_0,\beta_0\}$. If we use diminishing stepsizes $\alpha_t = \alpha_0/t$ and $\beta_t = \beta_0/t$ for GDA, then after $T$ steps we have $\|\theta^t\|\le T^{c_0\ell}L_0/\ell$ for $t\in [T]$.
\end{Lemma}

Therefore, if $C_e$ is much larger than $C_p$, using GDA with $C_p\le T^{c_0\ell}L_0/\ell\le C_e$ is better than GDMax. We make this more concrete in the context of GAN training next.

\subsection{GAN training}

We now study the specific case of GAN training to explore why GDA might generalize  better than GDMax. This is numerically  verified in the literature, such as \cite{farnia2021train}.   Specifically, we revisit Example \ref{example:gan}, and  consider a special case: $D$ is restricted to be a over-parametrized linear function with respect to $\theta$.  Define the descriminator  $D(x)=\Phi^T(x)v+b_0$, where $\Phi(x)=\left[\Phi_1(x),\cdots,\Phi_m(x)\right]^T\in \mathbb{R}^{m}$ is the feature matrix and $b_0\in \mathbb{R}$. 
Also suppose that $G$ is parametrized by $w$ and $G^*=G_{w^*}$. Then the GAN problem can be written as $\min_{w\in W}\max_{\theta\in \Theta}~r(w,\theta),$ where  
$$r(w,\theta)=E_{x\sim P_r}[\phi(v^T \Phi(x)+b_0)]+E_{y\sim P_0}[\phi(1-v^T\Phi(G_w(y))-b_0)].$$
Here $\theta = (v, b_0)$.
Assume that $\sqrt{\sigma_{\max }\left(E_{x\sim P_{G_w}}\Phi(x)\Phi^T(x)\right)}\le \bar{\sigma}_{\max}/\sqrt{m}$, where $\sigma_{\max}(\cdot)$ denotes the largest singular value of a matrix and $\bar{\sigma}_{\max}>0$ is a constant. Also assume that $E_{x\sim P_{G_w}}\Phi(x)\Phi^T(x)$ is full rank. 
Also, 
we assume that $|\phi'(\lambda)|\le L_\phi$ for any $\lambda\in [0,1]$. Therefore, we have $E[\|\nabla_{\theta} f(w,\theta;z)\|^2]\approx L_{\phi} ^2\bar{\sigma}^2_{\max}$.
Then it is reasonable to assume that $\|\nabla f\|\le \mathcal{O}(1)$. 

\begin{Lemma}\label{capacity_gan}
Suppose $\Phi(x)$ is sub-Gaussian and the matrix $$Q_S = \begin{bmatrix} \Phi(x_1)&\Phi(x_2)\cdots &\Phi(x_n)&\Phi(G_w(y_1))&\cdots &\Phi(G_w(y_n)) \end{bmatrix}$$ 
is full column rank ($m>n$) with probability $1$. Then with probability at least $1-C\delta$ with some constant $C$, we have $\|\theta_S(w^*)\|\ge  \Omega(\sqrt{n})$, where $\theta_S(w^*)\in \arg\max_{\theta\in \Theta}r_S(w^*,\theta)$.
\end{Lemma}

Now, for $\theta \in \arg\max_{\theta' \in \Theta}r(w^*,\theta')$, it can be easily seen that $v=0,b_0=1/2$ in this case. Therefore, $C_p\approx 1/2$. Finally, combining the previous discussion on GDA in Lemma \ref{gda_bounded}, and using the fact that $C_e$ is large from Lemma \ref{capacity_gan}, we see from Theorem \ref{populated_primal_gap} that GDA can generalize better than GDMax. More detailed discuss of the GAN-training example and Lemma \ref{capacity_gan} can be found in Section \ref{appendix:gda_gdmax}. 

%




\section{Conclusions}\label{sec:conclusion}

In this paper, we first demonstrate the shortcomings of {one popular metric, the primal risk, in terms of   characterizing the generalization behavior of minimax learners.} 
We then propose a new metric, the primal gap, {whose generalization error}  overcomes these shortcomings and captures the generalization behavior of algorithms that solve stochastic  minimax problems. Finally, we use this newly proposed metric to study the generalization behavior of two different algorithms -- GDA and GDMax, and study cases where GDA has a better generalization behavior than GDMax. {Future directions include further investigation of the proposed new metric, the primal gap, and deriving its (tighter) generalization error bounds in other structured stochastic minimax optimization problems in machine learning.} 


{\small
\bibliographystyle{plainnat}   
\bibliography{refs} }



\appendix


%



\section{Existing Related Results}

From \citep{farnia2021train}, we have the following theorem showing the connection between stability and generalization for minimax problems.

\begin{Theorem}[\citep{farnia2021train}]
\label{thm:farnia_1}
Consider an Algorithm $A$ which is $\epsilon$-stable. We have the following two claims:
\begin{enumerate}
\item If the maximization and the expectation can be swapped when computing $r(w)$, then $$E_SE_A[\zeta_{gen}^P(A)]\le \epsilon.$$ 
\item If $f(\cdot,\cdot;z)$ is nonconvex-strongly-concave and $f$ is $\mu$-strongly-concave with respect to $\theta$, then $$E_SE_A[\zeta_{gen}^P(A)]\le L\sqrt{\kappa^2+1}\epsilon.$$
\end{enumerate}
\end{Theorem}

\begin{Remark}
In \citep{lei2021stability}, the authors proved a generalization bound in a weak sense, i.e., they consider the weak duality gap:
$$(\max_{\theta\in \Theta}E_SE_Ar(w_S^A,\theta)-\min_{w\in W}E_SE_Ar(w,\theta_S^A))-(\max_{\theta\in \Theta}E_SE_Ar_S(w_S^A,\theta)-\min_{w\in W}E_SE_Ar_S(w,\theta_S^A)).$$
However, notice that the expectation is inside the min and max operators. It does not deal with the coupling of the maximization and expectation.
\end{Remark}
\begin{Remark}
According to Theorem \ref{thm:farnia_1}, the generalization bound for $\zeta_{gen}^P$ scales with the condition number $\kappa_\theta$, and therefore cannot give useful bounds in the absence of strong concavity (when $\kappa_\theta \rightarrow \infty$).
\end{Remark}

\begin{Remark}
The generalization bounds for  $\zeta_{gen}^P$ of algorithms for problems in terms of stability without strong concavity is still open to the best of our knowledge. As mentioned in \citep{lei2021stability}, finding generalization bounds without the strong concavity assumption is an interesting open problem.
\end{Remark}

\section{Analysis of Example \ref{example:main_ex}}
\label{sec:example_analysis}

In this section, we analyze the toy example given in Example \ref{example:main_ex}.

\begin{Proposition}
For the risk function and data distribution given in Example \ref{example:main_ex}, we have
$$E_S[r(w)-r_S(w)]\le 0$$
for any $w\in W$.
\end{Proposition}
\begin{proof}
For a fixed $w$, $r(w)=w^2/2-w$. 
On the other hand, 
\begin{eqnarray}
r_S(w)&=&\max_{\theta\in \Theta}r(w,\theta)\\
&\ge& r_S(w,0)\\
&=&r(w).
\end{eqnarray}
Therefore, we have the desired result.
\end{proof}
Next, we prove that $|\sum_{i=1}^nz_i|$ will stay in the interval $[0.5, \lambda]$ with high probability.
\begin{Lemma}\label{interval}
For large enough $\lambda>2$, we have
$$\mathrm{Pr}\bigg(\bigg|\sum_{i=1}^nz_i\bigg|\in [0.5,\lambda]\bigg)>0.4,\qquad \mathrm{Pr}\bigg(\bigg|\sum_{i=1}^nz_i\bigg|\in [2,\lambda]\bigg)>0.01.$$
\end{Lemma}
\begin{proof}
Let $y_i\sim N(0,1/\sqrt{n}),i=1,\cdots,n$ be $n$ i.i.d. variables. Then $\sum_{i=1}^ny_i\sim N(0,1)$.
According to the table of Normal distribution, we have $\mathrm{Pr}(|\sum_{i=1}^ny_i|\in [0.5,\lambda])\geq 0.41$.
By the definition of $z_i$, we have 
$$\mathrm{Pr}(|\sum_{i=1}^nz_i|\in [0.5,\lambda])\geq \mathrm{Pr}(|\sum_{i=1}^ny_i|\in [0.5,\lambda],|y_i|<3\log n/\sqrt{n})+\mathrm{Pr}(\max_{i\in [n]}(|y_i|)\geq 3\log n/\sqrt{n}).$$
For the first term, we have
\begin{align*}
&\mathrm{Pr}(|\sum_{i=1}^ny_i|\in [0.5,\lambda],|y_i|<3\log n/\sqrt{n})\\
&\geq \mathrm{Pr}(|\sum_{i=1}^ny_i|\in [0.5,\lambda])-\mathrm{Pr}(\max_{i\in [n]}(|y_i|)\ge 3\log n/\sqrt{n})\\
&\geq 0.41-\sum_{i=1}^n\mathrm{Pr}(|y_i|\geq 3\log n/\sqrt{n})\\
&\geq 0.41-ne^{-\gamma 9\log^2n}\geq 0.41-1/n^{\lambda\gamma-1}.
\end{align*}
Taking $\lambda$ sufficiently large yields the desired result, where the first inequality is because of the union bound and the second inequality is due  to the tail bound of Normal distribution.
Therefore, $\mathrm{Pr}(|\sum_{i=1}^nz_i|\in [0.5,\lambda])>0.4$ for sufficiently large $n$. The second statement follows similarly, noting from the table of Normal distribution that $\mathrm{Pr}(|\sum_{i=1}^ny_i|\in [0.5,\lambda])\geq 0.046$.  
\end{proof}
\begin{Proposition}
For sufficiently large $\lambda>0$, we have
$$E_S[r(w_S)-\min_{w\in W}r(w)]\ge 0.001.$$ 
\end{Proposition}

\begin{proof}
If $|\sum_{i=1}^nz_i|\in  [0.5,\lambda]$, we have
$$w_S=\max(0,1-(\sum_{i=1}^nz_i)^2/2)\le 0.9.$$
In this case, we have
\begin{equation}\label{0.05}
r(w_S)-\min_{w\in W}r(w)\ge 0.005,
\end{equation}
by direct calculation.
Therefore, we have
\begin{eqnarray}
&&E_S[r(w_S)-\min_{w\in W}r(w)]\\
&\ge& \mathrm{Pr}(|\sum_{i=1}^nz_i|\in [0.5,\lambda])\cdot 0.05+\mathrm{Pr}(|\sum_{i=1}^nz_i|\notin [0.5,\lambda])\cdot 0\\
&\ge&0.02,
\end{eqnarray}
where the first inequality is because of \eqref{0.05} and the fact that $r(w_S)-\min_{w\in W}r(w)\ge 0$ for any $S$.
\end{proof}
\begin{Proposition}\label{example-min-gap}
For sufficiently large $\lambda>0$,  we have:  
$$E_S[\min_{w\in W}r_S(w)-\min_{w\in W}r(w)]\ge 0.005$$
for Example \ref{example:main_ex}.
\end{Proposition}
\begin{proof}
If $|\sum_{i=1}^nz_i|\ge \lambda>2$, we have
$w_S=0$ and hence $r_S(w_S)=0$.
If $|\sum_{i=1}^nz_i|\le \lambda$, we have
$$r_S(w_S)-r(w^*)\ge r_S(w_S)-r(w_S)=w_S(\sum_{i=1}^nz_i)^2/2\ge 0.$$
Therefore, $\min_{w\in W}r_S(w)\ge \min_{w\in W}r(w)$ for any $S$. 
By  Lemma \ref{interval}, we can prove that $\mathrm{Pr}(|\sum_{i=1}^nz_i|\in [2,\lambda])\ge 0.01$ for sufficiently large $\lambda$.
Notice that for $|\sum_{i=1}^nz_i|\in [2,\lambda]$, $r_S(w_S)-\min_{w\in W}r(w)=1/2$. Therefore, we have
$$E_S[\min_{w\in W}r_S(w)-\min_{w\in W}r(w)]\ge \mathrm{Pr}(|\sum_{i=1}^nz_i|\in [2,\lambda])\cdot 1/2\ge 0.005.$$
This completes the proof. 
\end{proof}

\section{Proofs in Section \ref{sec:metric}}
\subsection{Proof of Lemma \ref{main1}}
In this subsection, we assume that $A$ is an $\epsilon$-stable algorithm. For any $w\in W$, let $\Theta_S(w)=\arg\max_{\theta\in \Theta}r_S(w,\theta)$ and $\Theta(w)=\arg\max_{\theta\in \Theta}r(w,\theta)$ be the solution sets of the problems. 
Let $\theta(w)$ be any element in $\Theta(w)$. 
Then
\begin{eqnarray*}
E_AE_S[r(w_S^A)-r_S(w_S^A)]&=&E_AE_S[r(w_S^A,\theta(w_S^A))-r_S(w_S^A,\theta_S(w_S^A))]\\
&\le&E_AE_S[r(w_S^A,\theta(w_S^A))-r_S(w_S^A,\theta(w_S^A))],
\end{eqnarray*}
where the inequality is because $r_S(w_S^A,\theta_S(w_S^A))\ge r_S(w_S^A,\theta)$ for any $\theta$.
Let $f$ be $\mu$-strongly concave with respect to $\theta$. We denote the condition number by $\kappa_\theta = \ell_{\theta\theta}/\mu$.

In the strongly concave case, $\Theta(w)$ has a unique element $\theta(w)$, which is $\kappa_\theta$-Lipschitz continuous with respect to $w$ (see \citep{lin2020gradient}).

Then, defining $\tilde{f}(w,z)=f(w,\theta(w);z)$, the minimax problem reduces to the usual minimization problem on the function $\tilde{f}$.  The stability and the Lipschitz continuity of $\theta(w)$ with respect to $w$ yield the generalization bound of $L\sqrt{\kappa^2+1}\epsilon$. This is the result shown in Theorem $1$ of \citep{farnia2021train}.

However, if the maximization problem is not strongly concave, we lose the Lipschitz continuity and the uniqueness.
To overcome this difficulty, we define an approximate maximizer $\bar{\theta}(w)$ to $r(w,\theta)$.
Concretely speaking, we define  $\bar{\theta}(w)$ to be the point after $s$ steps of gradient ascent for the function $r(w,\cdot)$ with a stepsize $1/\ell_{\theta\theta}$ and being  initialized at $0$.
Then we have the following lemma:

\begin{Lemma}\label{approximate1}
For any $w \in W$, we have\footnote{\label{footnote:ref}For point 2, it holds when $s > 0$. For $s = 0$, we have the bound $r(w)-r(w,\bar{\theta}(w))\le  \ell_{\theta\theta} C_p^2$. We do not separate this degenerate case for ease of presentation.}
\begin{enumerate}
\item $\|\bar{\theta}(w)-\bar{\theta}(w')\|\le s \frac{\ell}{\ell_{\theta\theta}}\|w-w'\|$.
\item $r(w)-r(w,\bar{\theta}(w))\le \ell_{\theta\theta} C_p^2/s$. 
\end{enumerate}
\end{Lemma}
\begin{proof}
To prove the first part, let $\theta_0=\theta_0'=0$. Define $\theta_t,\theta_t'$ recursively as follows:
$$\theta_{t+1}=\theta_t+\nabla_\theta r(w,\theta_t)/\ell_{\theta\theta}$$
and
$$\theta_{t+1}'=\theta'_t+\nabla_\theta r(w',\theta_t')/\ell_{\theta\theta}.$$
We prove $\|\theta_t-\theta_t'\|\le t \frac{\ell}{\ell_{\theta\theta}} \|w-w'\|$ by induction.
For $t=0$, $\theta_0-\theta_0'=0$. Assume the induction hypothesis $\|\theta_{t-1}-\theta_{t-1}'\|\le (t-1) \frac{\ell}{\ell_{\theta\theta}}\|w-w'\|$ holds. We have
\begin{align*}
\|\theta_t-\theta_t'\| &=\|(\theta_{t-1}+\nabla_\theta r(w,\theta_{t-1})/\ell_{\theta\theta})-(\theta_{t-1}'+\nabla_\theta r(w,\theta_{t-1}')/\ell_{\theta\theta}) \\
& \qquad \qquad +(\nabla_\theta r(w,\theta_{t-1}')-\nabla_\theta r(w',\theta_{t-1}'))/\ell_{\theta\theta}\| \\
&\le\|(\theta_{t-1}+\nabla_\theta r(w,\theta_{t-1})/\ell_{\theta\theta})-(\theta_{t-1}'+\nabla_\theta r(w,\theta_{t-1}')/\ell_{\theta\theta})\| \\
& \qquad \qquad +\|(\nabla_\theta r(w,\theta_{t-1}')-\nabla_\theta r(w',\theta_{t-1}'))/\ell_{\theta\theta}\| \\
&\le\|\theta_{t-1}-\theta_{t-1}'\|+\ell \|w-w'\|/\ell_{\theta\theta}\\
&\le (t-1) \frac{\ell}{\ell_{\theta\theta}}\|w-w'\|+\frac{\ell}{\ell_{\theta\theta}}\|w-w'\|\\
&= t\frac{\ell}{\ell_{\theta\theta}}\|w-w'\|,
\end{align*} 
where the first inequality follows from the triangle inequality, the second inequality follows from non-expansiveness of gradient ascent for concave functions and the $\ell$-Lipschitz continuity of $\nabla r$, and the third inequality follows from the induction hypothesis.

Therefore, letting $t=s$ completes the proof of the first part. The second part of this lemma is just the convergence result for gradient ascent on smooth concave functions (see e.g., \citep{nesterov2013introductory}).
\end{proof}
Consider a virtual algorithm $\bar{A}$: for any $S$, the algorithm returns $w=w_S^A$ and $\theta=\bar{\theta}(w_S^A)$.

\begin{Lemma}\label{virtual1}
The stability of this virtual algorithm is
$\epsilon \sqrt{\left( {s \frac{\ell}{\ell_{\theta\theta}}} \right)^2+1}$.
\end{Lemma}
\begin{proof}
It is direct from the first part of Lemma \ref{approximate1}.
\end{proof}
Then we have the generalization  bound of $r_S(w,\theta)$:
\begin{Lemma}\label{easy-generalization1}
We have
$$E_SE_A[r(w_S^A,\bar{\theta}(w_S^A))-r_S(w_S^A,\bar{\theta}(w_S^A))]\le \epsilon L\sqrt{\left( {s \frac{\ell}{\ell_{\theta\theta}}} \right)^2+1}.$$
\end{Lemma}
\begin{proof}
For any $z$, by Assumption \ref{Lipschitz-continuous},  we have
$$\|f(w_S^{\bar{A}},\theta_S^{\bar{A}};z)-f(w_{S'}^{\bar{A}},\theta_{S'}^{\bar{A}};z)\|\le \epsilon L\sqrt{\left( {s \frac{\ell}{\ell_{\theta\theta}}} \right)^2+1}.$$ 
The result follows directly from the standard stability theory in \citep{hardt2016train}.
\end{proof}

Now we are ready to derive the generalization {error} bound of the Primal Risk for an Algorithm $A$ with $\epsilon$-stability.
First, we have
\begin{align*}
E_SE_A[r(w_S^A)-r_S(w_S^A)] &\le E_SE_A[r(w_S^A)-r_S(w_S^A,\bar{\theta}(w_S^A))]\\
&\le E_SE_A[(r(w_S^A,\bar{\theta}(w_S^A)+\ell_{\theta\theta} C_p^2/s)-r_S(w_S^A,\bar{\theta}(w_S^A))]\\
&= E_SE_A[r(w_S^A,\bar{\theta}(w_S^A))-r_S(w_S^A,\bar{\theta}(w_S^A))]+\ell_{\theta\theta} C_p^2/s\\
&\le \epsilon L\sqrt{\left( {s \frac{\ell}{\ell_{\theta\theta}}} \right)^2+1}+\ell_{\theta\theta} C_p^2/s \\
&\leq \epsilon L{s \frac{\ell}{\ell_{\theta\theta}}} + \frac{\ell_{\theta\theta} C_p^2}{s} + \epsilon L 
\end{align*}
where the first inequality is because $r_S(w_S^A)=\max_{\theta}r_S(w_S^A,\theta)$, the second inequality is because of the second part of Lemma \ref{approximate1} and the last inequality is because of Lemma \ref{easy-generalization1}.
Optimizing over\footnote{Here we assume that the optimal $s$ is a real number greater than $0$. Constraining $s$ to be an integer and also incorporating $0$ does not change the result and we ignore this case here. See also Footnote \ref{footnote:ref}.} $s$, the generalization error  is bounded by $\zeta_{gen}^P(A)\le \sqrt{4L\ell C_p^2}\cdot \sqrt{\epsilon} + \epsilon L$. This  completes  the proof. \hfill\qed

\subsection{Proof of Theorem \ref{main3}}
Recall that the empirical primal gap is defined as
$$\Delta_S(w)=r_S(w)-\min_{w\in W}r_S(w)$$
and the  population primal gap is given by
$$\Delta(w)=r(w)-\min_{w\in W}r(w).$$

Suppose we are given an $\epsilon$-stable Algorithm $A$. We then want to derive the generalization error 
$$\zeta_{gen}^{PG}(A) = E_SE_A[\Delta(w_S^A)-\Delta_S(w_S^A)].$$

Since we already have the generalization error  for the primal risk $E_SE_A[r(w_S^A)-r_S(w_S^A)]$ in Theorem \ref{main1}, we only need to estimate
$$E_SE_A[\min_{w\in W}r_S(w)-\min_{w\in W}r(w)]=E_S[\min_{w\in W}r_S(w)-\min_{w\in W}r(w)]$$ 
to get a generalization {error} bound on the primal gap.

\begin{Lemma}\label{min-gap}
Let $w^*\in \arg\min_{w\in W}r(w)$. Suppose that $f(w^*,\cdot;z)$ is $L_\theta^*$ Lipschitz continuous with respect to $\theta$. Then we have
$$E_S[\min_{w\in W}r_S(w)-\min_{w\in W}r(w)]\le 4L_\theta^* C_e/\sqrt{n}.$$
\end{Lemma}
\begin{proof}
We use similar techniques as in the proof of Lemma \ref{main1}. 

\textbf{Step 1.} We define an approximate maximizer $\tilde{\theta}_S$ of the function $r_S(w^*,\cdot)$.
$\tilde{\theta}_S$ is attained by performing $s$ steps of gradient ascent  to $r_S(w^*,\cdot)$ with stepsize $1/\ell_{\theta\theta}$ and being initialized at $0$.

Similar to Lemma \ref{approximate1}, we have the following lemma:
\begin{Lemma}\label{approximate2}
We have the following properties:
\begin{enumerate}
\item $\|\tilde{\theta}_S-\tilde{\theta}_{S'}\|\le 2sL_\theta^*/(n\ell_{\theta\theta})$.
\item $r_S(w^*)-r_S(w^*,\tilde{\theta}_S)\le \ell_{\theta\theta} C_e^2/s$.
\end{enumerate}
\end{Lemma}
\begin{proof}
The proof is similar to the proof of Lemma \ref{approximate1}. To prove the first part, let $\tilde{\theta}_0=\tilde{\theta}_0'=0$. Define $\tilde{\theta}_t,\tilde{\theta}_t'$ recursively as follows:
$$\tilde{\theta}_{t+1}=\tilde{\theta}_t+\nabla_\theta r_S(w^*,\tilde{\theta}_t)/\ell_{\theta\theta}$$
and
$$\tilde{\theta}_{t+1}'=\tilde{\theta}'_t+\nabla_\theta r_{S'}(w^*,\tilde{\theta}_t')/\ell_{\theta\theta}.$$
We prove $\|\tilde{\theta}_t-\tilde{\theta}_t'\|\le L_\theta^*/(n\ell_{\theta\theta})$ by induction.
For $t=0$, $\tilde{\theta}_0-\tilde{\theta}_0'=0$. Assume the induction hypothesis $\|\tilde{\theta}_{t-1}-\tilde{\theta}_{t-1}'\|\le (t-1) L_\theta^*/(n\ell_{\theta\theta})$ holds. We have
\begin{align*}
\|\tilde{\theta}_t-\tilde{\theta}_t'\| &=\|(\tilde{\theta}_{t-1}+\nabla_\theta r_S(w^*,\tilde{\theta}_{t-1})/\ell_{\theta\theta})-(\tilde{\theta}_{t-1}'+\nabla_\theta r_S(w^*,\tilde{\theta}_{t-1}')/\ell_{\theta\theta}) \\
& \qquad \qquad +(\nabla_\theta r_{S}(w^*,\tilde{\theta}_{t-1}')-\nabla_\theta r_{S'}(w^*,\tilde{\theta}_{t-1}'))/\ell_{\theta\theta}\| \\
&\le\|(\tilde{\theta}_{t-1}+\nabla_\theta r_S(w^*,\tilde{\theta}_{t-1})/\ell_{\theta\theta})-(\tilde{\theta}_{t-1}'+\nabla_\theta r_S(w^*,\tilde{\theta}_{t-1}')/\ell_{\theta\theta})\| \\
& \qquad \qquad +\|(\nabla_\theta r_{S}(w^*,\tilde{\theta}_{t-1}')-\nabla_\theta r_{S'}(w^*,\tilde{\theta}_{t-1}'))/\ell_{\theta\theta}\| \\
&\le\|\tilde{\theta}_{t-1}- \tilde{\theta}_{t-1}'\|+\ell \|w-w'\|/\ell_{\theta\theta}\\
&\le (t-1)\frac{2L_\theta^*}{n\ell_{\theta\theta}} +\frac{2L_\theta^*}{n\ell_{\theta\theta}}\\
&= t\frac{2L_\theta^*}{n\ell_{\theta\theta}},
\end{align*} 
where the first inequality follows from the triangle inequality, the second inequality follows from non-expansiveness of gradient ascent for concave functions and the $L_{\theta}^*$-Lipschitz continuity of $ f(w^*, \cdot; z)$, and the third inequality follows from the induction hypothesis.

Therefore, letting $t=s$ completes the proof of the first part. The second part of this lemma is just the convergence result for gradient ascent on smooth concave functions (see e.g., \citep{nesterov2013introductory}).
\end{proof}
We then define the virtual algorithm $\tilde{A}$ given by $w_S^{\tilde{A}}=w^*$ and $\theta_S^{\tilde{A}}=\tilde{\theta}_S$.
Since the output argument $w$ of $\tilde{A}$ is always  $w^*$, the stability of $\tilde{A}$ only depends on $\tilde{\theta}_S$.
Then the stability bound of this virtual algorithm is given in the following lemma:
\begin{Lemma}
The stability of Algorithm $\tilde{A}$ is given by $\epsilon_{sta}(\tilde{A})=2s{(L_\theta^*)^2}/(n\ell_{\theta\theta})$.
\end{Lemma}
Then by the standard stability theory in \citep{hardt2016train}, we have
\begin{equation}\label{tildeoriginal}
|E_SE_A[r_S(w^*,\tilde{\theta}_S)-r(w^*,\tilde{\theta}_S)]|\le 2s(L_\theta^*)^2/(n\ell_{\theta\theta}).
\end{equation}
\textbf{Step 2.}
We have
\begin{align*}
E_S[\min_{w\in W}r_S(w)-\min_{w\in W}r(w)] &\stackrel{ \mbox{\scriptsize(i)}}= E_S[r_S(w_S)-r(w^*,\theta^*)]\\
&\stackrel{ \mbox{\scriptsize(ii)}}\le E_S[r_S(w^*)-r(w^*,\theta^*)]\\
&\stackrel{ \mbox{\scriptsize(iii)}}\le E_S[r_S(w^*,\tilde{\theta}_S)-r(w^*,\theta^*)]+\ell_{\theta\theta} C_e^2/s\\
&\stackrel{ \mbox{\scriptsize(iv)}}\le E_S[r_S(w^*,\tilde{\theta}_S)-r(w^*,\tilde{\theta}_S)]+\ell_{\theta\theta} C_e^2/s,
\end{align*}  
where (i) follows from the definition of $w^*,\theta^*$, (ii) follows since $w_S$ minimizes  $r_S(w)$,  (iii) follows from Lemma \ref{approximate2}, and (iv) follows from the optimality of $\theta^*$ given $w^*$.
Then by \eqref{tildeoriginal}, we have
\begin{align}
E_S[\min_{w\in W}r_S(w)-\min_{w\in W}r(w)] &\le E_S[r_S(w^*,\tilde{\theta}_S)-r(w^*,\tilde{\theta}_S)]+\ell_{\theta\theta} C_e^2/s\\
&\le 2s (L_\theta^*)^2/(n\ell_{\theta\theta})+\ell_{\theta\theta} C_e^2/s\\
&\le 4L_\theta^* C_e/\sqrt{n}
\end{align}
which completes the proof. 
\end{proof}
The final statement of the theorem follows from Lemma \ref{min-gap} and Lemma \ref{main1}. \hfill \qed
  
\subsection{Proof of Lemma \ref{sufficient}}

We only prove the first part of this lemma and the others can be proved  similarly.
Let $s=[D_p/\gamma]+1$, where $[r]$ denotes the largest integer no more than $r$.
To prove the first part, let $\theta_0=\theta_0'=0$. Define $\theta_t,\theta_t'$ recursively as follows:
$$\theta_{t+1}=\theta_t+c_0\nabla_\theta r(w,\theta_t)/t$$
and
$$\theta_{t+1}'=\theta'_t+c_0\nabla_\theta r(w',\theta_t')/t.$$
We prove $\|\theta_t-\theta_t'\|\le t \frac{\ell}{\ell_{\theta\theta}} \|w-w'\|$ by induction.
For $t=0$, $\theta_0-\theta_0'=0$. Assume the induction hypothesis $\|\theta_{t-1}-\theta_{t-1}'\|\le (t-1) \frac{\ell}{\ell_{\theta\theta}}\|w-w'\|$. We have
\begin{align}
\|\theta_t-\theta_t'\| &=\|(\theta_{t-1}+c_0\nabla_\theta r(w,\theta_{t-1})/t)-(\theta_{t-1}'+c_0\nabla_\theta r(w,\theta_{t-1}')/t) \\
& \qquad \qquad +c_0(\nabla_\theta r(w,\theta_{t-1}')-\nabla_\theta r(w',\theta_{t-1}'))/t\| \\
&\le\|(\theta_{t-1}+c_0\nabla_\theta r(w,\theta_{t-1})/t)-(\theta_{t-1}'+c_0\nabla_\theta r(w,\theta_{t-1}')/t)\| \\
& \qquad \qquad +c_0\|(\nabla_\theta r(w,\theta_{t-1}')-\nabla_\theta r(w',\theta_{t-1}'))/t\| \\
&\le(1+c_0\ell_{\theta\theta}/t)\|\theta_{t-1}-\theta_{t-1}'\|+c_0\ell \|w-w'\|/t.
\end{align}
Here the first inequality follows from the triangle inequality, the second inequality follows from the $\ell_{\theta \theta}-$Lipschitz continuity of $\nabla_\theta r$ and $\ell$-Lipschitz continuity of $\nabla r$.
Therefore, we have
$$\|\theta_t-\theta_t'\|\le (1+c_0\ell_{\theta\theta}/t)\|\theta_{t-1}-\theta_{t-1}'\|+c_0\ell \|w-w'\|/t.$$
Let $\delta_{t}=\|\theta_t-\theta_t'\|$. Then by the above recursion, we have
$$\delta_t+\ell/\ell_{\theta\theta} \|w-w'\|\le \prod_{i=1}^{t}(1+c_0\ell_{\theta\theta}/i)\ell \|w-w'\|/\ell_{\theta\theta}.$$
Using the inequalities $e^a\ge 1+a$ and $\sum_{i=1}^t 1/i\le \log t$, we have
$$\delta_t\le \frac{t\ell}{\ell_{\theta\theta}}\|w-w'\|.$$ 
Letting $t=s$ yields 
$$\|\theta_p^{\gamma}(w)-\theta_p^{\gamma}(w')\|\le \frac{s\ell}{\ell_{\theta\theta}}\|w-w'\|.$$
Since $D_p>\gamma$, we have
$$s\le [D_p/\gamma]+1\le 2D_p/\gamma.$$
Hence, $$s\frac{\ell}{\ell_{\theta\theta}}\cdot \gamma\le 2D_p\ell/\ell_{\theta\theta}.$$
Setting $\lambda_p=2D_p\ell/\ell_{\theta\theta}$ yields the desired result.

\subsection{Proof of Lemma \ref{ncncp}}

This is similar to the proof of Lemma \ref{main1}. We first define the virtual algorithm $\bar{A}$ which outputs $(w_S^A,\theta_p^{\gamma}(w_S^A))$.
By Assumption \ref{oracle1}, it can be easily seen that $\bar{A}$ is $(1+\lambda_p/\gamma)\epsilon$-stable.
Then by Theorem \ref{hardt}, we have
$$E_SE_A[r(w_S^A,\theta_p^{\gamma}(w_S^A))-r_S(w_S^A,\theta_p^{\gamma}(w_S^A))]\le L(1+\lambda_p/\gamma)\epsilon.$$
This gives us:
\begin{eqnarray*}
E_SE_A[r(w_S^A)-r_S(w_S^A)]&\le& E_SE_A[r(w_S^A,\theta_p^{\gamma}(w_S^A))-r_S(w_S^A,\theta_p^{\gamma}(w_S^A))]+\gamma\\
&\le&L\epsilon+L\lambda_p\epsilon /\gamma+\gamma.
\end{eqnarray*}
Taking $\gamma=\sqrt{L\lambda_p} \sqrt{\epsilon}$, we have
$$\zeta_{gen}^p(A)\le L\epsilon+\sqrt{L\lambda_p}\sqrt{\epsilon}.$$

\begin{algorithm*}[t]
	\caption{GDA}
	\begin{algorithmic}[1]
		\REQUIRE initial iterate $(w_S^0, \theta_S^0) = (0,0)$, stepsizes $\alpha_t, \beta_t$, projection operators $P_W$ and $P_\Theta$;
		\FOR{$t = 0,\dots,T-1$}
		\STATE $w_S^{t+1} = P_W \left( w_S^t - \alpha_t \nabla_w r_S(w, \theta) \right)$
		\STATE$ \theta^{t+1}_S = P_\Theta \left(  \theta^t_S + \beta_t  \nabla_\theta r_S(w, \theta) \right)$
		\ENDFOR
	\end{algorithmic}
	\label{alg_gda}
\end{algorithm*}

\begin{algorithm*}[t]
	\caption{GDMax}
	\begin{algorithmic}[1]
		\REQUIRE initial iterate $(w_S^0, \theta_S^0) = (0,0)$, stepsizes $\alpha_t$, projection operators $P_W$ and $P_\Theta$;
		\FOR{$t = 0,\dots,T-1$}
		\STATE $w_S^{t+1} = P_W \left( w_S^t - \alpha_t \nabla_w r_S(w, \theta) \right)$
		\STATE$ \theta^{t+1}_S = \underset{\theta \in \Theta} {\operatorname{argmax}} \ r_S(w_S^{t+1}, \theta)$
		\ENDFOR
	\end{algorithmic}
	\label{alg_gdmax}
\end{algorithm*}

\section{Proofs in Section \ref{sec:comparison}}\label{appendix:gda_gdmax}

\subsection{Proof of Theorem \ref{populated_primal_gap}}
First, we have
\begin{align}
&E_SE_{A_0}[r(w_S^{A_0,T}) -\min_{w\in W}r(w)] \nonumber \\
&=E_SE_{A_0}[r_S(w_S^{A_0,T})-\min_{w\in W}r_S(w)]+E_SE_{A_0}[r(w_S^{A_0,T})-r_S(w_S^{A_0,T})] \nonumber \\
& \qquad \qquad +E_SE_{A_0}[\min_{w\in W}r_S(w)-\min_{w\in W}r(w)].
\end{align}
Furthermore, by Assumption \ref{convergence_GDA} and Theorem \ref{main3}, we have
$$E_SE_{A_0}[r(w_S^{A_0,T})-\min_{w\in W}r(w)]\le 
(\phi_{A_0}(M_w) + \phi_{A_0}(C_e(\Theta_0)))/\psi_{A_0}(T)+\zeta_{gen}^P(A_0) +L_\theta^*C_e(\Theta_0)/\sqrt{n}.$$
Next, notice that the output of $A_0$ is equal to the output of $A$ with probability at least $1-\delta$ and $ \| r(w) \| \le 1$. Therefore, we have
$$|E_SE_A[r(w_S^{A,T})]-E_SE_{A_0}[r(w_S^{A_0,T})]|\le \delta,$$ 
which gives the desired result. \hfill \qed

\subsection{Proof of Lemma \ref{lemma:gda_stab}}
Define $\delta_t=\|(w_S^t,\theta_S^t)-(w_{S'}^t,\theta_{S'}^t)\|$.
We have
\begin{align}
\delta_{t+1} &\le (1+c_0\ell/t)\delta_t+ 2c_0L_{\Theta_\theta^{GDA}}/nt.  \nonumber 
\end{align}
Therefore, 
\begin{align}
\delta_{t+1} + \frac{2 L_{\Theta_\theta^{GDA}}}{ \ell n} &\leq  (1+c_0\ell/t) \left( \delta_t +  \frac{2 L_{\Theta_\theta^{GDA}}}{ \ell n} \right) \leq \frac{2 L_{\Theta_\theta^{GDA}}}{ \ell n} T^{c_0 \ell},
\end{align}
which completes the proof. \hfill \qed

\subsection{Proof of Lemma \ref{gda_bounded}}
For a fixed dataset $S$, let $g_t=\nabla r_S(w^t,\theta^t)$ and $d_t=\|(w^0,\theta^0)-(w^t,\theta^t)\|$.
Then we have 
$g_t\le L_0+d_t\ell$
and
$d_{t+1}\le d_t+c_0g_t/t$.
Substituting the first inequality into the second one, we have
$$d_{t+1}\le d_t+c_0d_t/t+L_0c_0/t,$$
which gives us
$$d_{t+1}+L/\ell\le (1+c_0\ell/t)(d_t+L_0/\ell).$$
Multiplying this inequality from $0$ to $T-1$ yields
$$d_T\le T^{c_0\ell}L_0/\ell,$$
which completes the proof. \hfill \qed

\begin{figure}[!t]
    \centering
    \includegraphics[width=0.3\linewidth]{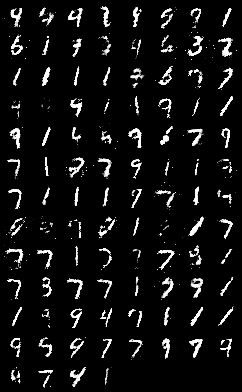}
    \qquad\qquad\qquad\includegraphics[width=0.3\linewidth]{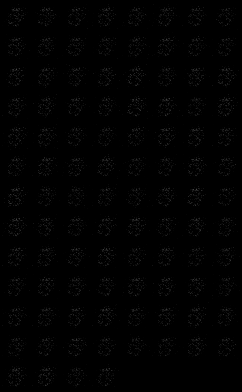}\\
    (a)~~GDA\qquad\qquad\qquad\qquad\qquad\qquad\qquad (b)~~GDMax 
    \caption{Comparison of the results on MNIST generated by GDA and GDMax.}
    \label{fig:gda_gdmax_MNIST_comp}
\end{figure}

\subsection{Proof of Lemma \ref{capacity_gan}}
Let $u=[1,1,\cdots,1,0,\cdots,0]^T\in \mathbb{R}^{2n}$. Then $\theta_S(w)$ satisfies $Q_S^T\theta_S(w)=u-b_0e$, where $e=[1,1,\cdots,1]^T\in \mathbb{R}^{2n}$. It can be easily seen that $\|u-b_0e\|\ge \sqrt{n}/2$.

We can also show that $\sigma_{\max}(Q_S)\le 2\sigma_{\max }\cdot \sigma_{\max} (P)$, where $P\in \mathbb{R}^{2n\times m}$ is full row-rank and independent rows. Moreover, every row of $P$ has covariance matrix $I_m/\sqrt{m}$.
Then by random matrix theory (see \citep{vershynin2010introduction}), we have $\sigma_{\max}(P)\le \mathcal{O}(\sqrt{m}/\sqrt{m}-C\sqrt{n}/\sqrt{m} + \log(1/\delta)/\sqrt{m})=\mathcal{O}(1)$ with probability $1 - C\delta$. Therefore, we have $\theta_{S}(w)\ge \Omega(\sqrt{n})$. \hfill \qed

\subsection{Experiments on GAN-training}\label{sec:exp_GAN}

In this section, we provide some numerical results to corroborate our theoretical findings.  

\subsubsection{Setup}

We train a GAN on MNIST data using two algorithms -- GDA and GDMax. Since the stability is improved by using adaptive methods like Adam, we use Adam-descent-ascent (ADA) and Adam-descent-max (ADMax) instead. ADA  simultaneously trains the generator and the discriminator, while ADMax  trains the optimal discriminator for each generator step. We simulate this by taking $10$ steps of ascent for every descent step.  Figure   \ref{fig:gda_gdmax_MNIST_comp} plots the images generated by GANs trained using these two algorithms. Finally,  in Figure \ref{fig:gda_gdmax_norm_comp}, we plot the norms of the discriminator trained by these two algorithms. 

\subsubsection{Results}

Figure  \ref{fig:gda_gdmax_MNIST_comp} plots the images generated by GANs trained using GDA and GDMax (using Adam instead of the simple gradient step). As predicted by the theory in Section \ref{appendix:gda_gdmax}, we can see that GDA produces better images than the corresponding GAN trained using GDMax. Furthermore, the claim that $C_e >> C_p$ can be seen from Figure  \ref{fig:gda_gdmax_norm_comp} where we see that the norm of the discriminator trained using GDMax is much larger than the norm of the discriminator trained using GDA. This follows from the results in Section \ref{subsec:C_e_C_p}. GDMax trains the discriminator to exactly distinguish between the empirical data generated by the true and fake distributions. Therefore, when they are nearly the same, their empirical distributions would be close as well. This would imply that the discriminator would need to have a very large slope (Lipschitz constant) to exactly distinguish between the two empirical datasets, and this in turn leads to a large discriminator norm (which captures the Lipschitz constant of the discriminator).

\begin{figure}[!t]
    \centering
    \includegraphics[width=0.6\linewidth]{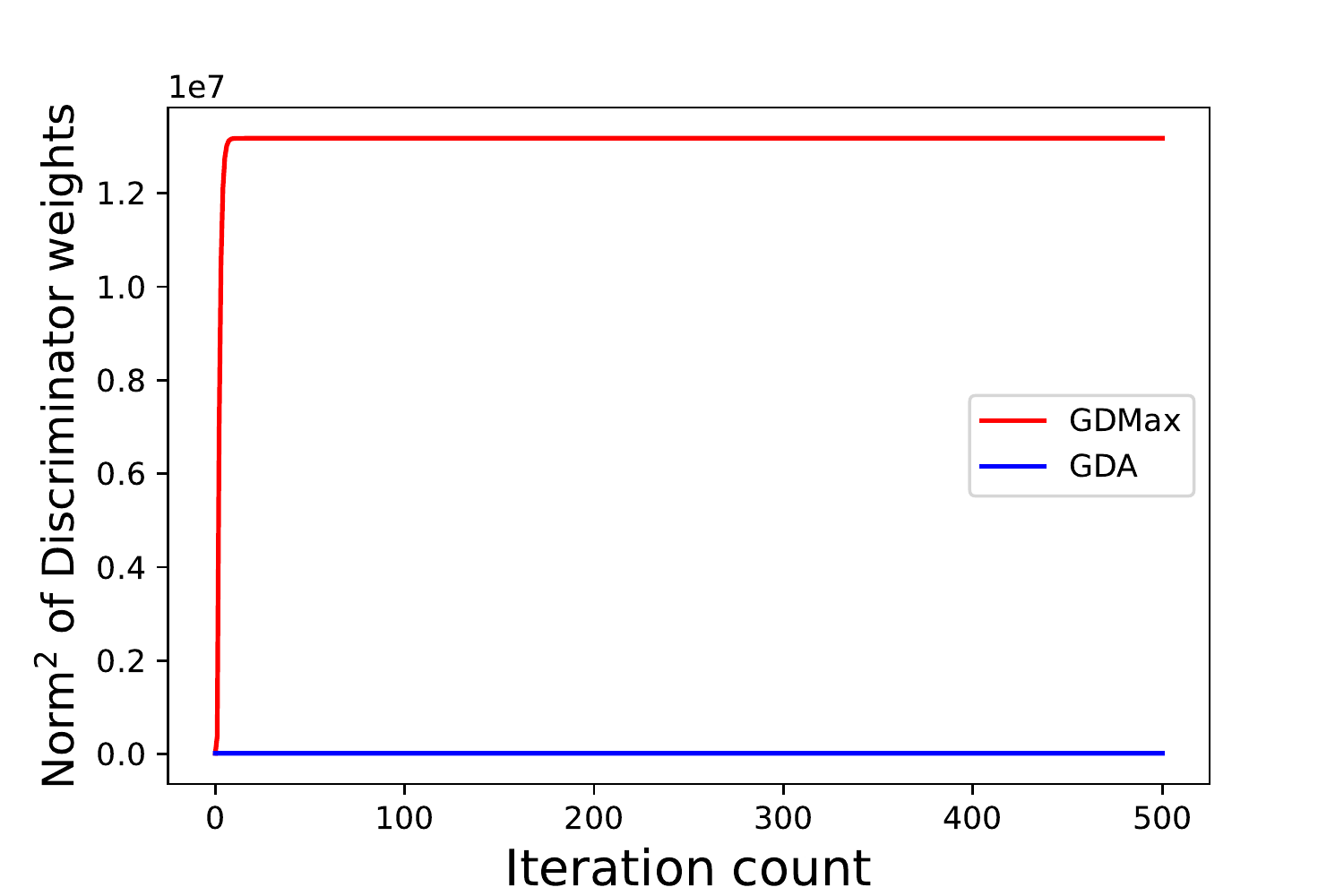}
    \caption{Comparison of the norm squares of discriminator weights.}
    \label{fig:gda_gdmax_norm_comp}
\end{figure}

%
%

 \section{Generalization Error for Primal-Dual  Risk} 
\label{sec:connection_to_pd_gap}

If the saddle-point exists, the {\it primal-dual  risk} is often a good measure of generalization: 

\begin{Definition}{[Primal-dual risk]}
The population and empirical primal-dual (PD) risks are defined as: 
$$\Delta^{PD}(w,\theta)=\max_{\theta'\in \Theta}r(w,\theta')-\min_{w'\in W}r(w',\theta),$$
and
$$\Delta_S^{PD}(w,\theta)=\max_{\theta'\in \Theta}r_S(w,\theta')-\min_{w'\in W}r_S(w',\theta).$$
\end{Definition}

A point $(w,\theta)$ is called a saddle-point of $r_S$ (or $r$) if $\Delta_S^{PD}(w,\theta)=0$ (or $\Delta^{PD}(w,\theta)=0$). Furthermore, if a saddle-point $(w_S,\theta_S)$ exists for $r_S(\cdot,\cdot)$, we have $w_S=\min_{w\in W}r_S(w)$.
Moreover, if $w_S\in \arg\min_{w\in W}r_S(w)$ and $\theta_S\in \arg\max_{\theta\in \Theta}r_S(w_S,\theta)$, then $(w_S,\theta_S)$ is a saddle point of $r_S(\cdot,\cdot)$.

Notice that  if we can find an approximate  saddle point $(w_S,\theta_S)$ of $r_S(w,\theta)$, i.e., $\Delta_S^{PD}(w_S,\theta_S) < \epsilon$ and guarantee that $\Delta^{PD}(w_S,\theta_S)-\Delta_S^{PD}(w_S,\theta_S)$ is small, we can guarantee that $\Delta(w_S,\theta_S)$ is small and therefore $(w_S,\theta_S)$ is an approximate saddle point of $r(\cdot,\cdot)$.
Hence if the saddle point exists for $r_S(\cdot,\cdot)$, the {generalization error of the primal-dual risk  can be a good measure for the generalization of the solution to the empirical problem.} We define the expected generalization error  for the primal-dual risk as follows:
\begin{Definition}
The generalization error  for the primal-dual risk is defined as $$\zeta_{gen}^{PD}(A)=E_SE_A[\Delta^{PD}(w_S^A,\theta_S^A)-\Delta_S^{PD}(w_S^A,\theta_S^A)].$$
\end{Definition}

\subsection{The generalization of the primal-dual risk for convex-concave problems}

Similar to Definition \ref{def:capacity}, we define the $W$-capacity as follows: 

\begin{Definition}[W-Capacity]
\label{def:Wcapacity}
Let 
\begin{align}
W^*(\theta)=\min_{w\in W}r(w,\theta) , \text{ and } \quad W_S(\theta)=\min_{w\in W}r_S(w,\theta). \nonumber
\end{align}
The $W$-capacities $C_e^w$ and $C_p^w$ are defined as
\begin{align}
C_p^w &= \max_{\theta}\mathrm{dist}(0,W^*(\theta)  \nonumber \\
C_e^w &= \max_{S,\theta}\mathrm{dist}(0,W_S(\theta)).
\end{align}
\end{Definition}
Next, we also define the following:
\begin{Definition}
Let $f^-(\theta,w;z)=-f(w,\theta;z)$. We first have
\begin{align}
r^-(\theta,w)=E_{z\sim P_z}[f^-(\theta,w;z)], \qquad r_S^-(\theta,w)=\frac{1}{n}\sum_{i=1}^nf^-(\theta,w;z_i). 
\end{align}
Furthermore, we define:
\begin{align}
r^-(\theta) &= \max_{w\in W}r^-(\theta,w)=-(\min_{w\in W}r(w,\theta)) \nonumber \\
r^-_S(\theta) &= \max_{w\in W}r^-_S(\theta,w)=-(\min_{w\in W}r_S(w,\theta)).
\end{align} 
\end{Definition}
Now, we have the following bound for the generalization error of the primal-dual risk, $\zeta_{gen}^{PD}(A)$ for an $\epsilon$-stable Algorithm  $A$:
\begin{Theorem}
\label{thm:pd_cc}
Suppose that Algorithm $A$ is  $\epsilon$-stable. The generalization error $\zeta_{gen}^{PD}(A)$ for convex-concave problem, i.e., when $f(\cdot, \cdot; z)$ is convex-concave for all $z$, is bounded by: 
$$\zeta_{gen}^{PD}(A) \leq  \left( \sqrt{4L\ell C_p^2} +  \sqrt{4L\ell (C_p^w)^2} \right) \sqrt{\epsilon} + 2\epsilon L.$$
\end{Theorem}
\begin{proof}
Notice that 
\begin{eqnarray}
\zeta_{gen}^{PD}(A)&=& E_SE_A[\Delta^{PD}(w_S^A,\theta_S^A)-\Delta_S^{PD}(w_S^A,\theta_S^A)]\\
&=&E_SE_A[r(w_S^A)-r_S(w_S^A)]+E_SE_A[r^-(\theta_S^A)-r^-_S(\theta_S^A)].
\end{eqnarray}
The two terms can be bounded by Lemma \ref{main1} respectively.
By Lemma \ref{main1}, we have
$$E_SE_A[r(w_S^A)-r_S(w_S^A)]\le \sqrt{4L\ell C_p^2} \sqrt{\epsilon} + \epsilon L$$
and
$$E_SE_A[r^-(\theta_S^A)-r^-_S(\theta_S^A)]\le \sqrt{4L\ell (C_p^w)^2} \sqrt{\epsilon} + \epsilon L.$$
Combining these two inequalities yields the desired result.
\end{proof}

\subsection{$\zeta_{gen}^{PD}(T)$ for the proximal point algorithm}


In this section, we study the generalization behavior  of the proximal point algorithm (PPA) ((See Equation (3) in \citep{farnia2021train})). By \citep{farnia2021train}, the stability of $T$ steps of PPA can be bounded as follows:
\begin{Lemma} [\citep{farnia2021train}]\label{lemma:PPA_stab}
The stability of $T$ steps of PPA can be bounded by $\epsilon \leq \mathcal{O}\left(T/n\right)$.
\end{Lemma}

Therefore, substituting the result of Lemma \ref{lemma:PPA_stab} in Theorem \ref{thm:pd_cc}, we have the following bound for $\zeta_{gen}$ for $T$ steps of PPA:
\begin{Theorem}\label{PPA_gen}
After $T$ steps of PPA, the generalization error of the primal-dual risk can be bounded by:
$$\zeta_{gen}^{PD}(T) \leq \mathcal{O}\left(\sqrt{T/n} + T/n \right).$$ 
\end{Theorem}
\subsection{The population primal-dual risk of PPA}

Finally, we give the population primal-dual risk after $T$ steps of PPA. By \citep{mokhtari2020convergence}, we  have the following convergence result of PPA.
\begin{Lemma}[\citep{mokhtari2020convergence}]\label{PPA_rate}
Let $(w_S^t, \theta_S^t)$ be the iterates obtained after $t$ iterations of proximal point algorithm on the function $r_S(\cdot, \cdot)$ and $\bar{w}_S^t=\frac{1}{t}\sum_{i=1}^tw_S^i, \bar{\theta}_S^t=\frac{1}{t}\sum_{i=1}^t\theta_S^i$ be the averaged iterates. Then we have
$$\Delta_S^{PD}(\bar{w}_S^T,\bar{\theta}_S^T)\le \ell(C_e^2 +(C_e^w)^2)/T.$$
\end{Lemma}  
Combining Lemma \ref{PPA_rate} and Theorem \ref{PPA_gen}, we have the following result:
\begin{Theorem}
Let $(w_S^t, \theta_S^t)$ be the iterates obtained after $t$ iterations of proximal point algorithm on the function $r_S(\cdot, \cdot)$ and $\bar{w}_S^t=\frac{1}{t}\sum_{i=1}^tw_S^i, \bar{\theta}_S^t=\frac{1}{t}\sum_{i=1}^t\theta_S^i$ be the averaged iterates.  Then, the expected population primal-dual risk at the point $(\bar{w}_S^t, \bar{\theta}_S^t)$ can be bounded by:
$$E_S[\Delta^{PD} (w_S^t, \theta_S^t)] \le \mathcal{O}\left(1 /T + \sqrt{T/n} + T/n \right).$$
\end{Theorem}

\end{document}